\documentclass[lettersize,journal]{IEEEtran}
\usepackage{amsmath,amsfonts}
\usepackage{algorithmic}
\usepackage{algorithm}
\usepackage{array}
\usepackage[caption=false,font=scriptsize,labelfont=sf,textfont=sf]{subfig}
\usepackage{textcomp}
\usepackage{stfloats}
\usepackage{url}
\usepackage{verbatim}
\usepackage{graphicx}
\usepackage{cite}
\usepackage{adjustbox} 
\usepackage[T1]{fontenc}
\usepackage[utf8]{inputenc}
\usepackage{amssymb,subfloat,booktabs,makecell,bbding,multirow,tabularx,ragged2e}

\usepackage[dvipsnames]{xcolor}
\RequirePackage[colorlinks,citecolor=RoyalBlue]{hyperref}

\usepackage{pifont}
\usepackage{tikz}

\newcommand{\onedot}{.}
\def\eg{\emph{e.g}\onedot} 
\def\ie{\emph{i.e}\onedot} 
\def\cf{\emph{cf}\onedot} 
\def\etc{\emph{etc}\onedot}

\def\etal{\emph{et al}\onedot}

\usepackage{bm}
\usepackage{mathtools, amsthm}
\usepackage{cleveref}

\usepackage{amsmath}
\usepackage{amssymb}
\usepackage{mathtools}
\usepackage{amsthm}

\usepackage{cleveref}

\usepackage{wrapfig}

\theoremstyle{plain}
\newtheorem{theorem}{Theorem}
\newtheorem{theoremappendix}{Theorem}

\theoremstyle{definition}
\newtheorem{definition}{Definition}

\theoremstyle{remark}

\usepackage{relsize}

\begin{document}

\newcommand{\mytitle}{Active Negative Loss: A Robust Framework for Learning with Noisy Labels}

\title{\mytitle}

\author{Xichen Ye,
        Yifan Wu, 
        Yiqi Wang,
        Xiaoqiang Li,~\IEEEmembership{Member,~IEEE},\\
        Weizhong Zhang,
        and~Yifan Chen,~\IEEEmembership{Member,~IEEE} 
\thanks{Xichen Ye, Yiqi Wang, and Xiaoqiang Li are with the School of Computer Engineering and Science, Shanghai University, Shanghai, China (e-mail: yexichen0930@shu.edu.cn; wangyq2004@shu.edu.cn; xqli@shu.edu.cn)}
\thanks{Yifan Wu is with the School of Computer Science, Fudan University, Shanghai, China (e-mail: victorwu001219@gmail.com).}
\thanks{Weizhong Zhang is with the School of Data Science, Fudan University, Shanghai, China, and the Shanghai Key Laboratory of Intelligent Information Processing (e-mail: weizhongzhang@fudan.edu.cn).}
\thanks{Yifan Chen is with the Departments of Computer Science and Math, Hong Kong Baptist University, Hong Kong, China (e-mail: yifanc@hkbu.edu.hk).}
\thanks{Xichen Ye and Yifan Wu contributed equally to this work.}
\thanks{Corresponding authors: Xiaoqiang Li and Yifan Chen.}
}

\markboth{Journal of \LaTeX\ Class Files,~Vol.~14, No.~8, August~2021}%
{Shell \MakeLowercase{\textit{et al.}}: A Sample Article Using IEEEtran.cls for IEEE Journals}

\IEEEpubid{0000--0000/00\$00.00~\copyright~2021 IEEE}
\maketitle

\begin{abstract}
Deep supervised learning has achieved remarkable success across a wide range of tasks, yet it remains susceptible to overfitting when confronted with noisy labels. 
To address this issue, noise-robust loss functions offer an effective solution for enhancing learning in the presence of label noise.
In this work, we systematically investigate the limitation of the recently proposed Active Passive Loss (APL), which employs Mean Absolute Error (MAE) as its passive loss function. 
Despite the robustness brought by MAE, one of its key drawbacks is that it pays equal attention to clean and noisy samples;
this feature slows down convergence and potentially makes training difficult, particularly in large-scale datasets. 
To overcome these challenges, we introduce a novel loss function class, termed \emph{Normalized Negative Loss Functions} (NNLFs), which serve as passive loss functions within the APL framework. 
NNLFs effectively address the limitations of MAE by concentrating more on memorized clean samples. 
By replacing MAE in APL with our proposed NNLFs, we enhance APL and present a new framework called \emph{Active Negative Loss} (ANL).
Moreover, in non-symmetric noise scenarios, we propose an entropy-based regularization technique to mitigate the vulnerability to the label imbalance.
Our extensive experiments across various types of label noise, including symmetric, asymmetric, instance-dependent, and real-world noise, demonstrate that the new loss functions adopted by our ANL framework can achieve better or comparative performance compared with the state-of-the-art methods.
In addition, we investigate the application of our ANL framework in image segmentation, further showcasing its superior performance.
The source code of our method is available at: \href{https://github.com/Virusdoll/Active-Negative-Loss}{https://github.com/Virusdoll/Active-Negative-Loss}.
\end{abstract}



\section{Introduction}
\label{sec:intro}

\IEEEPARstart{R}{ecently}, deep neural networks (DNNs) have made significant strides in a variety of supervised learning tasks, ranging from simple tasks such as image classification~\cite{krizhevsky2009learning, imagenet, he2016deep} to more complex tasks such as image segmentation~\cite{long2015fullySemanticSegmentation, miccai2015unet, he2017maskrcnn}. 
This remarkable achievement is mainly driven by the proficiency of DNNs in capturing subtle and intricate representations within images. 
However, such performance heavily relies on large-scale datasets with precise annotations (often referred to as ``clean'' labels), which are both expensive and time-consuming to obtain. 
Furthermore, the labeling process for large-scale datasets is prone to errors, inevitably leading to ``noisy'' (mislabeled) samples due to human bias or lack of expert knowledge in annotation \cite{han2020survey}. 
Additionally, previous studies have shown that over-parameterized DNNs possess the capacity to fit even completely random labels~\cite{zhang2017understanding, arpit2017closer, jiang2018mentornet}, which suggests that their evaluation performance tends to degrade significantly when trained on noisy data. 
As a result, learning with noisy labels has garnered considerable attention.

To address this challenge, researchers have proposed various strategies to mitigate the adverse effects of noisy labels, and one popular research line is to design noise-robust loss functions, which is the main focus of this paper as well.
\cite{ghosh2017robust}~has theoretically proved that, \textit{symmetric} loss functions (\cf~\Cref{def:sym}) such as Mean Absolute Error (MAE), are inherently robust to noise, while commonly used loss functions like Cross Entropy (CE) are not.
However,
MAE treats all samples equally, generating identical loss gradients with respect to the model’s output probabilities, and thus overly weighs a sample with a high-confidence prediction in training.
This limitation hinders effective learning and suggests that MAE may not be suitable for training DNNs on
complicated datasets~\cite{zhang2018generalized}. 
In response to these drawbacks, several works proposed \emph{partially robust loss functions}, including Generalized Cross Entropy~\cite[GCE]{zhang2018generalized}, a generalized mixture of CE and MAE, and Symmetric Cross Entropy (SCE) \cite{wang2019symmetric}, a combination of CE and a scaled MAE, Reverse Cross Entropy (RCE).

\textbf{\IEEEpubidadjcol}
Recently, Active Passive Loss \cite[APL]{ma2020normalized} framework has been proposed to create fully robust loss functions,
which
shows that any loss function can be made robust to noisy labels by a simple normalization operation.
Specifically, to address the underfitting problem of normalized loss functions, APL first categorizes existing loss functions into two types and then combines them.
The two types are: \ding{182} ``Active'' loss, which only corresponds to the explicit maximization of the probability that a sample is in the labeled class, and \ding{183} ``Passive'' loss, which explicitly minimizes the probabilities of being in other classes.

However, after investigating several robust loss functions under the APL framework, we find that their passive loss functions within are \textbf{always} scaled versions of MAE,
which has been considered training-unfriendly in the previous discussion.
In summary, while APL intended to address underfitting by combining active loss functions with MAE, it inadvertently inherits the training inefficiencies associated with MAE, which may hinder
the model performance and make training more challenging.

In this paper, we propose a new class of passive loss functions beyond MAE, called \emph{Negative Loss Functions} (NLFs).
In constructing NLFs, we combine \ding{182} complementary label learning~\cite{ishida2017learning, yu2018learning} and \ding{183} a simple ``vertical flipping'' operation, and consequently any active loss function can be converted into a passive loss function.
Moreover, to make it theoretically robust to noisy labels, we further apply the normalization operation to NLFs for obtaining \emph{Normalized Negative Loss Functions} (NNLFs).
Through replacing the MAE in APL with NNLF, we coin a novel framework called \emph{Active Negative Loss} (ANL).
ANL combines a normalized active loss functions and an NNLF to build a new set of noise-robust loss functions, which can be seen as an improvement of APL.
We show that under our proposed ANL framework, several commonly-used loss functions can be made robust to noisy labels while ensuring sufficient fitting ability to achieve state-of-the-art performance for training DNNs with noisy datasets.


Moreover, by examining the marginal probabilities for each class in non-symmetric label noise scenarios, we identify the problem of class imbalance.
Specifically, under non-symmetric label noise, each class or sample is assigned non-uniform probabilities of being mislabeled as other classes, resulting in a dataset with imbalanced class distributions.
This imbalance leads to a biased model and degrades its performance.
To mitigate the drawback, we propose an entropy-based regularization approach that steers the model output marginal probabilities to be more balanced.
Experimental results in \Cref{experiments} demonstrate the effectiveness of this technique.

Before delving into details, we highlight the main contributions of this paper from four aspects:
\begin{itemize}
    \item We present an approach to build a new class of robust passive loss functions called \emph{Normalized Negative Loss Function}s (NNLFs).
    By replacing the MAE in APL with our proposed NNLFs, we propose a new loss framework \emph{Active Negative Loss} (ANL).
    \item We analyze label imbalance in non-symmetric label noise scenarios and propose an entropy-based regularization technique to alleviate it.
    \item We theoretically prove that our proposed NNLFs and ANL are robust to noisy labels and discuss the reason why replacing the MAE in APL with our NNLFs will improve performance in noisy label learning.
    \item We conduct extensive experiments on various datasets under various label noise scenarios and empirically demonstrate that the new set of loss functions, developed under our proposed ANL framework, outperforms state-of-the-art methods in both image classification and segmentation tasks.
\end{itemize}
A short version of this manuscript was presented in Neurips 2023 \cite{ye2023active}, where we introduced Active Negative Loss framework for classiﬁcation with noisy labels. 
This paper encloses a comprehensive extension of the conference version;
specifically, \ding{182}~we enhance our approach by introducing an entropy-based regularization technique to address label imbalance;
\ding{183}~additionally, we expand our experiments to cover a broader range of label noise scenarios, including instance-dependent and real-world label noise.
\ding{184}~Furthermore, we extend our evaluation to the image segmentation task, demonstrating the effectiveness of our approach beyond the conventional classification tasks.

The rest of this paper is organized as follows. 
Section~\ref{related} briefly reviews the related work of learning with noisy labels in deep learning. 
Section~\ref{preliminaries} details the preliminaries of this work, including its theoretical basis and motivation.
In Section~\ref{anlf}, we present our proposed ANL framework and NNLFs, along with their theoretical justification.
We validate the efficiency of our method over a range of experiments in Section~\ref{experiments}. 
Section~\ref{conclusion} concludes the paper.

\section{Related Work} 
\label{related}

In this section, we review the related work of this paper, including the works on robust loss function, sample selection, and label correction.

\subsection{Robust Loss Function and Symmetry}
\label{sec:related-robust loss}

Ghosh \etal \cite{ghosh2017robust} have theoretically proved that Cross Entropy (CE) and some other popular loss functions are not robust to label noise because they are \textbf{asymmetric}.
On the other hand, symmetric loss functions like Mean Absolute Error (MAE) are robust to noise, but it takes a relatively longer time to converge.

In order to combine the robustness of MAE and the fast convergence speed of CE, several partially robust loss functions are proposed.
Notable examples include Generalized Cross Entropy (GCE)\cite{zhang2018generalized}, a generalization of MAE and CE, and Symmetric Cross Entropy (SCE)\cite{Wang-2019-SCE}, a combination of scaled MAE and CE.
PHuber-CE\cite{Menon-2020-CanGradClip} further introduces composite-loss gradient clipping, while OGC\cite{DBLP:conf/aaai/OGC} improves this idea by learning an optimized clipping threshold, all of which can be viewed as hybrids that interpolate between CE and MAE.
Furthermore, Feng \etal \cite{Feng-2020-TaylorCE} leveraged the Taylor series to derive an alternative formulation that interpolates between CE and MAE, resulting in the Taylor Cross Entropy (TCE).

Recently, Ma \etal \cite{Ma-2020-APL} proposed Active Passive Loss (\textbf{APL}) framework to create fully robust loss functions,
which categorizes existing loss functions into two types, `Active' type and `Passive' type, and then combines them.
Zhou \etal \cite{Zhou-2023-asymmetric} investigated the effectiveness of asymmetric loss in addressing label noise and proposed Asymmetric Loss Functions (ALFs).
There are also methods enlightened by complementary label learning, such as NLNL\cite{Kim-2019-NLNL} and JNPL\cite{Kim-2021-JNPL}, which use complementary labels to reduce the exposure of the model under wrong information.

\subsection{Other Methods}

We review two other important techniques: 
(i) sample selection and label correction, and 
(ii) Co-teaching frameworks.

\subsubsection{Sample Selection and Label Correction}
Another line of research focuses on identifying or correcting noisy samples. 
Sample-selection methods typically perform sample re-weighting to down-weight unreliable instances\cite{Song-2022-noisy-survey,song2019selfie,Wang-2018-impor-rewei,Chang-2017-active-bias}. 
This category also includes data pruning\cite{cho2025dual}, which filters out potentially noisy samples, though its pruning ratio often requires scenario-specific tuning and can be inefficient in practice. 
Representative selection approaches include MentorNet~\cite{jiang2018mentornet}, which learns a curriculum to guide the student model. 
Label-correction methods instead aim to revise noisy annotations directly~\cite{yu2024fatal,tanaka2018joint}, such as Meta Label Correction~\cite{zheng2017metaLC}, noise-channel models~\cite{Bekker2016unreliable,Liu2024temporal,Misra2016humancentr} or influence-based methods~\cite{ye2025towards}. 

\subsubsection{Co-teacher Framework}
Co-teaching–based methods mitigate the confirmation bias caused by noisy labels by training multiple peer networks that exchange potential knowledge during learning\cite{han2018co}. 
This mutual-teaching mechanism helps each model avoid overfitting to the specific noise patterns memorized by the other. 
Recent advances propose asymmetric variants to further improve robustness, such as ACT~\cite{sheng2024enhancing}, CA2C~\cite{CA2C}, \etc
Although effective, these approaches generally require training and maintaining at least two models simultaneously, resulting in increased computational cost.


\section{Preliminaries} 
\label{preliminaries}

In this section, we first present the formal definition of the classification problem under label noise in Section~\ref{sec: def_RM}. 
We then introduce the label noise model in Section~\ref{sec: label_noise_model}, discussing different types of noise, including symmetric, asymmetric, and instance-dependent noise.
Lastly, we review the recent state-of-the-art Active Negative Loss Functions in Section~\ref{sec: APL}.

\subsection{Risk Minimization under label noise}
\label{sec: def_RM}

Consider a typical K-class classification problem.
Let $\mathcal{X} \subset \mathbb{R}^d$ be the $d$-dimensional feature space from which the samples are drawn, and $\mathcal{Y} = [k] = \{1,\cdots,K\}$ be the label space.
Given a clean training dataset, $\mathcal{S} = \{(\boldsymbol{x}_n,y_n)\}_{n=1}^N$, where each $(\boldsymbol{x}_n,y_n)$ is drawn \emph{i.i.d.} from an unknown distribution, $\mathcal{D}$, over $\mathcal{X} \times \mathcal{Y}$.
We denote the ground-truth distribution over different labels for sample $\boldsymbol{x}$ by $\boldsymbol{q}(k|\boldsymbol{x})$, and $\sum_{k=1}^K \boldsymbol{q}(k|\boldsymbol{x})=1$.
The distribution is singular: since there is only one corresponding label $y$ for a $\boldsymbol{x}$, we have $\boldsymbol{q}(y|\boldsymbol{x})=1$ and $\boldsymbol{q}(k \ne y|\boldsymbol{x}) = 0$.

A classifier is defined as $h(\boldsymbol{x})= \arg \max_i f(\boldsymbol{x})_i$, where $f: \mathcal{X} \to \mathcal{C}$, $\mathcal{C} \subseteq [0, 1]^K$ is a simplex (i.e., $\forall \boldsymbol{c} \in \mathcal{C}$, $\boldsymbol{1}^T \boldsymbol{c} = 1$), mapping feature space to label space.
In this work, $f$ is a DNN with a softmax output layer.
For each sample $\boldsymbol{x}$, $f(\boldsymbol{x})$ computes its probability estimation $\boldsymbol{p}(k|\boldsymbol{x})$ for each class $k \in \{1,\cdots,K\}$, and $\sum_{k=1}^K \boldsymbol{p}(k|\boldsymbol{x})=1$.
Throughout this paper, we refer to $f(\cdot)$ itself as the classifier.
Training a classifier $f(\cdot)$ is to find an optimal classifier $f^*(\cdot)$ that minimizes the empirical risk: $\sum_{n=1}^N \mathcal{L}(f(\boldsymbol{x}_n), y_n)$, in which $\mathcal{L}: \mathcal{C} \times \mathcal{Y} \to \mathbb{R}^+$ is a loss function, and $\mathcal{L}(f(\boldsymbol{x}), k)$ is the loss of $f(\boldsymbol{x})$ with respect to label $k$.

When label noise exists, the model can only access a corrupted dataset $\mathcal{S}_\eta = \{(\boldsymbol{x}_n,\hat{y}_n)\}^N_{n=1}$, where each sample is drawn \emph{i.i.d.} from an unknown distribution $\mathcal{D}\eta$.
In this context, minimizing the empirical risk $\sum_{n=1}^N \mathcal{L}(f(\boldsymbol{x}_n), \hat{y}_n)$ produces a classifier \(\hat{f}^*(\cdot)\) optimized on the corrupted data.
Our goal in noisy label learning is to find a loss function \(\mathcal{L}\) that minimizes the discrepancy between the risks of \(f^*(\cdot)\) and \(\hat{f}^*(\cdot)\).

\subsection{Label Noise Modeling}
\label{sec: label_noise_model}

To model the label noise, we follows convention and formulates the noisy label $\hat{y}$ as ($y$ is the ground-truth label):
\begin{equation}
    \hat{y} = \left\{
    \begin{aligned}
    & y & \text{with probability} & \ (1 - \eta_y), \\
    & j, j \in [k], j \ne y & \text{with probability} & \ \eta_{yj},
    \end{aligned}
    \right.
\end{equation}
where $\eta_{yj}$ denotes the probability that true label $y$ is corrupted into label $j$, and $\eta_y=\sum_{j \ne y}\eta_{yj}$ denotes the noise rate of label $y$.
This model follows the \emph{class-dependent} label noise assumption, where the noise transition probability depends on the sample class or target class.
Additionally, under this class-dependent assumption, we can distinguish two specific types of label noise: symmetric and asymmetric.

For \emph{symmetric} label noise \cite{natarajan2013learning}, the corruption rate is uniform, meaning $\eta_{ij}=\frac{\eta_i}{K-1}, \forall j \ne y$ and $\eta_i = \eta, \forall i \in [k]$, where the whole modeling is decided by a constant $\eta$. 
In this case, each label has an equal chance of being corrupted into any other label.

For \emph{asymmetric} label noise \cite{ma2020normalized, zhou2021asymmetric}, the noise rate $\eta_{ij}$ depends on both the true label $i$ and the corrupted label $j$. 
This typically reflects more realistic scenarios where certain classes are more likely to be confused with others (\eg, automobiles and trucks).

On top of that, we also consider the \emph{instance-dependent} label noise assumption, which has gained increasing attention in recent years, as it better simulates real-world scenarios.
This assumption extends the previous setting by positing that the corruption process depends not only on the true label \(y\) and corrupted label \(\hat{y}\), but also on the input sample \(\bm{x}\) \cite{DBLP:conf/nips/XiaL0WGL0TS20:PDN}.

\subsection{Active Passive Loss Functions}
\label{sec: APL}

Ma \etal \cite{ma2020normalized} proposed the concept of normalized loss functions, where a given loss function $\mathcal{L}$ can be normalized as follows:
\begin{equation}
  \mathcal{L}_{\text{norm}}(f(\bm x), y; \mathcal{L}) = \frac{\mathcal{L}(f(\boldsymbol{x}),y)}{\sum^K_{k=1}\mathcal{L}(f(\boldsymbol{x}),k)}.  
\end{equation}
This simple normalization operation can make any loss function robust to noisy labels.
For instance, the Normalized Cross Entropy (NCE) is defined as (note CE amounts to the negative log likelihood):
\begin{equation}
    \mathcal{L}_\text{NCE}(f(\bm x),y) = \frac{\sum_{k=1}^K \boldsymbol{q}(k|\boldsymbol{x})(- \log \boldsymbol{p}(k|\boldsymbol{x}))}{\sum_{j=1}^K \sum_{k=1}^K \boldsymbol{q}(y=j|\boldsymbol{x}) \log \boldsymbol{p}(k|\boldsymbol{x})}.
\end{equation}
Similarly, Focal Loss (FL), Mean Absolute Error (MAE), and Reverse Cross Entropy (RCE) can be normalized to obtain Normalized Focal Loss (NFL), Normalized Mean Absolute Error (NMAE), and Normalized Reverse Cross Entropy (NRCE), respectively.

However, a normalized loss function alone suffers from underfitting.
To address this issue, Ma \etal \cite{ma2020normalized} categorized loss functions into two types: \emph{Active} and \emph{Passive}.

They first rewrite the loss function $\mathcal{L}(f(\boldsymbol{x}), y)$ in an expectation form, such that $\mathcal{L}(f(\boldsymbol{x}), y) = \sum_{k=1}^K \ell(f(\boldsymbol{x}), k)$ (\eg, let $\mathcal{L}$ be CE, then $\mathcal{L}(f(\boldsymbol{x}),y)$ $=$ $\sum_{k=1}^K \boldsymbol{q}(k|\boldsymbol{x})(- \log \boldsymbol{p}(k|\boldsymbol{x}))$, and $\ell (f(\boldsymbol{x}),k) = \boldsymbol{q}(k|\boldsymbol{x})(- \log \boldsymbol{p}(k|\boldsymbol{x}))$). Based on this, the two types of loss functions are defined as follows:
\begin{definition}[Active loss function]
\label{def:active}
    $\mathcal{L}_\text{Active}$ is an active loss function if $\forall (\boldsymbol{x}, y) \in \mathcal{D}, \forall k \ne y, \ell(f(\boldsymbol{x}), k) = 0$.
\end{definition}
\begin{definition}[Passive loss function]
\label{def:passive}
    $\mathcal{L}_\text{Passive}$ is a passive loss function if $\forall (\boldsymbol{x}, y) \in \mathcal{D}, \exists k \ne y, \ell(f(\boldsymbol{x}),k) \ne 0$.
\end{definition}
Active loss functions focus solely on explicitly maximizing the posterior probability estimation $\boldsymbol{p}(y|\boldsymbol{x})$ for the label $y$, since for the other labels the corresponding gradient remains zero. 
In contrast, passive loss functions not only maximize the probability at $k = y$ but also substantially minimize the probabilities at $k \neq y$.
Accordingly, active loss functions include CE, FL, NCE, and NFL, while passive loss functions include MAE, RCE, NMAE, and NRCE. 

These two types of loss functions can complement each other to alleviate underfitting.
Therefore, Ma \etal~\cite{ma2020normalized} proposed the Active Passive Loss (APL):
\begin{equation}
    \mathcal{L}_{\text{APL}} = \alpha  \mathcal{L}_{\text{Active}} + \beta  \mathcal{L}_{\text{Passive}},
\end{equation}
where $\alpha, \beta >0$ are hyperparameters.
Specifically, by combining NCE and RCE, Ma \etal ~introduced NCE+RCE, which is considered one of the state-of-the-art methods.

\section{Our Approach}
\label{anlf}

In this section, we detail the proposed Active Negative Loss (ANL) framework for classification with noisy labels. 
First, we present a limitation of APL in section~\ref{sec: limitation_APL} to offer insights. 
Next, we introduce the proposed ANL framework which integrates the general active loss functions with the proposed Normalized Negative Loss Functions (NNLFs) in Section~\ref{sec: detail_ANLF}. 
Following that, we propose entropy regularization to address the issue of imbalanced label under non-symmetric label noise
in Section~\ref{sec: entropy_regularize}.

\subsection{Limitations of APL with MAE}
\label{sec: limitation_APL}

As discussed in Section~\ref{sec: APL}, there are four specific passive loss functions remodeled in APL, namely MAE, RCE, NMAE, and NRCE. 
Nevertheless,
all of these passive loss functions are a multiple of MAE, 
and APL can thus be rewritten as:
\begin{equation}
    \mathcal{L}_{\text{APL}} = \alpha  \mathcal{L}_{\text{Active}} + \beta  \mathcal{L}_\text{MAE}.
\end{equation}
This indicates that MAE is an essential feature of APL. 

However, we recall \Cref{sec:related-robust loss} stated MAE tends to require longer training time and even makes learning difficult.
As a result, while APL relies on passive loss functions to address the underfitting of active part, MAE's training inefficiency may hinder APL's overall performance. 
This motivates us to explore new robust passive loss functions.

\subsection{Active Negative Loss Framework}
\label{sec: detail_ANLF}

In response to the limitations of APL, we propose the Active Negative Loss (ANL) framework, which retains the active loss functions from APL but replaces the original passive loss functions with Normalized Negative Loss Functions (NNLFs).

\subsubsection{Normalized Negative Loss Functions}
\label{sec: NNLFs}

The design of the Normalized Negative Loss Functions is guided by three essential objectives:

\begin{itemize}
    \item The loss function should optimize the classifier's output probability for at least one class position that is not specified by the given label.
    \item The loss function should focus on minimizing the classifier’s output probabilities for incorrect labels, rather than maximizing for the correct label alone.
    \item The loss function must be robust to noisy labels, ensuring stable learning despite label corruption.
\end{itemize}

To achieve these objectives, we introduce a method to derive NNLFs from existing active loss functions. This method incorporates three key components: 
\ding{182}~complementary label learning, 
\ding{183}~``vertical flipping'', and 
\ding{184}~normalization. Each of these components addresses the corresponding objectives, ensuring that NNLFs effectively mitigate the challenges posed by noisy labels.
We illustrate these components one by one as follows.

\ding{182}~Complementary label learning \cite{ishida2017learning, yu2018learning} is an indirect learning method for training CNNs. 
In this method, each input is assigned a complementary label, indicating a class that the input does not belong to. 
Unlike conventional training approaches, complementary label learning emphasizes the classifier's predictions for complementary labels, making it naturally suited for passive loss functions in which the loss $\ell(f(\boldsymbol{x}),k)$ for the complementary label $k$ is non-zero (\cf~\Cref{def:passive}). 
In this context, we adopt the core concept of this method by focusing the loss function on all classes $\{1, \cdots, K\}$ except the true labeled class $y$.

\begin{figure}[t]
\vskip -0.0in
\begin{center}
\centerline{\includegraphics[width=0.37\textwidth]{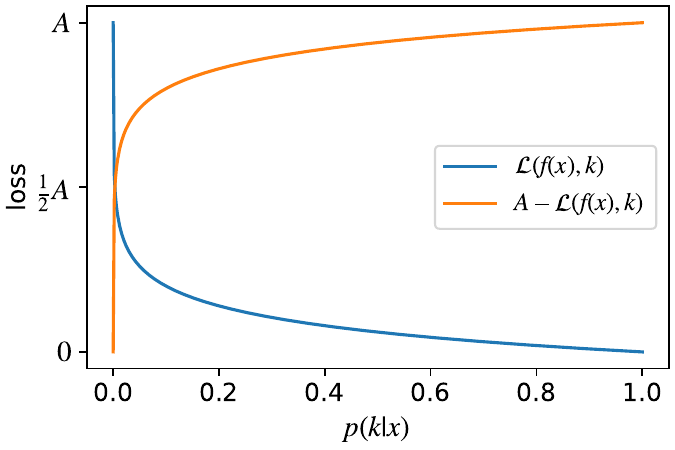}}
\vskip -0.0in
\caption{``Vertical flipping'' operation. $\mathcal{L}(f(\boldsymbol{x}),k)$ is an active loss function, \eg, CE. Loss function $A - \mathcal{L}(f(\boldsymbol{x}),k)$ is obtained by flipping $\mathcal{L}(f(\boldsymbol{x}),k)$ vertically with axis $\text{loss}=\frac{1}{2}A$.}
\label{vertical-flipping}
\end{center}
\vspace{-3mm}
\end{figure}

\ding{183}~``Vertical flipping'' is a straightforward operation that converts a loss function from ``maximizing'' to ``minimizing''.
As shown in the Figure~\ref{vertical-flipping}, given an active loss function $\mathcal{L}(f(\boldsymbol{x}),k)$, the new loss function $A - \mathcal{L}(f(\boldsymbol{x}),k)$ is obtained by flipping $\mathcal{L}(f(\boldsymbol{x}),k)$ vertically with axis $\text{loss}=\frac{1}{2}A$.
Notably, $A - \mathcal{L}(f(\boldsymbol{x}),k)$ is the opposite of $\mathcal{L}(f(\boldsymbol{x}),k)$, effectively shifting the focus toward minimizing the output probability $\boldsymbol{p}(k|\boldsymbol{x})$ to $0$.

Based on these two components (the incorporation of \ding{184} will be deferred to \hyperlink{norm}{the paragraph} two paragraphs down), given an active loss function \(\mathcal{L}\), we can already propose Negative Loss Functions (NLFs), an intermediate form, as follows (note $\boldsymbol{q}(k|\boldsymbol{x})$ is the ground-truth probability and $\boldsymbol{q}(y|\boldsymbol{x}) = 1$):
\begin{equation}
    \mathcal{L}_\text{neg} (f(\boldsymbol{x}),y;\mathcal{L}) = \sum_{k=1}^K \left(1-\boldsymbol{q}(k|\boldsymbol{x})\right) \big(A - \mathcal{L}(f(\boldsymbol{x}),k)\big). 
    \label{eq6}
\end{equation}
Here, $A$ is a constant representing the maximum loss value of $\mathcal{L}$, which ensures that the loss $\mathcal{L}_\text{neg}$ does not fall below zero (for the sake of sensible normalization).
However, this can lead to computational issues during implementation.
For instance, if $\mathcal{L}$ is the cross-entropy loss (CE), then $A = -\log 0 = +\infty$.
To address this, we clip the minimum probability to $p_\text{min} = 1 \times 10^{-7}$ in practice.

Our proposed NLF can make any active loss function into a passive loss function, where a) $1-\boldsymbol{q}(k|\boldsymbol{x})$ ensures that the loss function focuses on classes $\{1,\cdots,K\} \setminus \{y\}$, and b) $A-\mathcal{L}(f(\boldsymbol{x}),k)$ ensures that the loss function minimizes the output probability $\boldsymbol{p}(k|\boldsymbol{x})$.
\hypertarget{norm}{~}

\ding{184}~Next, in order to make our proposed passive loss functions robust to noisy labels, we perform the normalization operation on NLFs.
Given an active loss function $\mathcal{L}$, we propose Normalized Negative Loss Functions (NNLFs) as follows:
\begin{equation}
\label{eq:nnlf}
    \mathcal{L}_\text{nn} (f(\boldsymbol{x}),y; \mathcal{L}) = 1 - \frac{A-\mathcal{L}(f(\boldsymbol{x}),y)}{\sum_{k=1}^K A - \mathcal{L} (f(\boldsymbol{x}),k)},
\end{equation}
where $A$ has the same definition as Eq.~\eqref{eq6}.
The detailed derivation of NNLFs can be found in Appendix~\ref{appendix: loss_functions}.
Also, NNLFs have this property: $\mathcal{L}_\text{nn} \in [0, 1]$.
Accordingly, we can create NNLFs from active loss functions as follows.

The Normalized Negative Cross Entropy (NNCE) is:
\begin{equation}
    \mathcal{L}_\text{NNCE}(f(\bm x), y)
    = 1 - \frac{
    A + \log{\boldsymbol{p}(y|\boldsymbol{x})}
    }{
    \sum^K_{k=1} A + \log{\boldsymbol{p}(k|\boldsymbol{x})}
    };
\end{equation}
the Normalized Negative Focal Loss (NNFL) is:
\fontsize{9.2pt}{12pt}
\begin{align}
    \mathcal{L}_\text{NNFL}(f(\bm x), y)
    = 1 - \frac{
    A + (1 - \boldsymbol{p}(y|\boldsymbol{x}))^\gamma \log{\boldsymbol{p}(y|\boldsymbol{x})}
    }{
    \sum_{k=1}^K A + (1 - \boldsymbol{p}(k|\boldsymbol{x}))^\gamma \log{\boldsymbol{p}(k|\boldsymbol{x})}
    }.
\end{align}
\normalsize
As a closing remark, we note compared to the intermediate form, the unnormalized NLFs, the loss functions proposed above avoid the access to the ground-truth probability $\boldsymbol{q}$. 

\subsubsection{ANL Framework}

We can now create new robust loss functions by replacing the MAE in APL with our proposed NNLF.
Given an active loss function $\mathcal{L}$, we propose Active Negative Loss (ANL) functions as follows:
\begin{equation}
    \mathcal{L}_{\text{ANL}}= \alpha  \mathcal{L}_\text{norm} + \beta  \mathcal{L}_\text{nn}.
\end{equation}

Here, $\alpha, \beta > 0$ are parameters, $\mathcal{L}_\text{norm}$ denotes the normalized $\mathcal{L}$ and $\mathcal{L}_\text{nn}$ denotes the Normalized Negative Loss Function corresponding to $\mathcal{L}$. Accordingly, we can create ANL from the two mentioned active loss functions as follows.

For Cross Entropy (CE), we have ANL-CE:
\begin{equation}
\label{eq:anl-ce}
    \mathcal{L}_{\text{ANL-CE}}=\alpha  \mathcal{L}_{\text{NCE}} + \beta  \mathcal{L}_{\text{NNCE}}.
\end{equation}

For Focal Loss (FL), we have ANL-FL:
\begin{equation}
    \mathcal{L}_{\text{ANL-FL}}=\alpha  \mathcal{L}_{\text{NFL}} + \beta  \mathcal{L}_{\text{NNFL}}.
\end{equation}

\subsection{Entropy Regularization Against Imbalanced Label}
\label{sec: entropy_regularize}

\begin{figure}[t]
\begin{center}
\centerline{\includegraphics[width=0.49\textwidth]{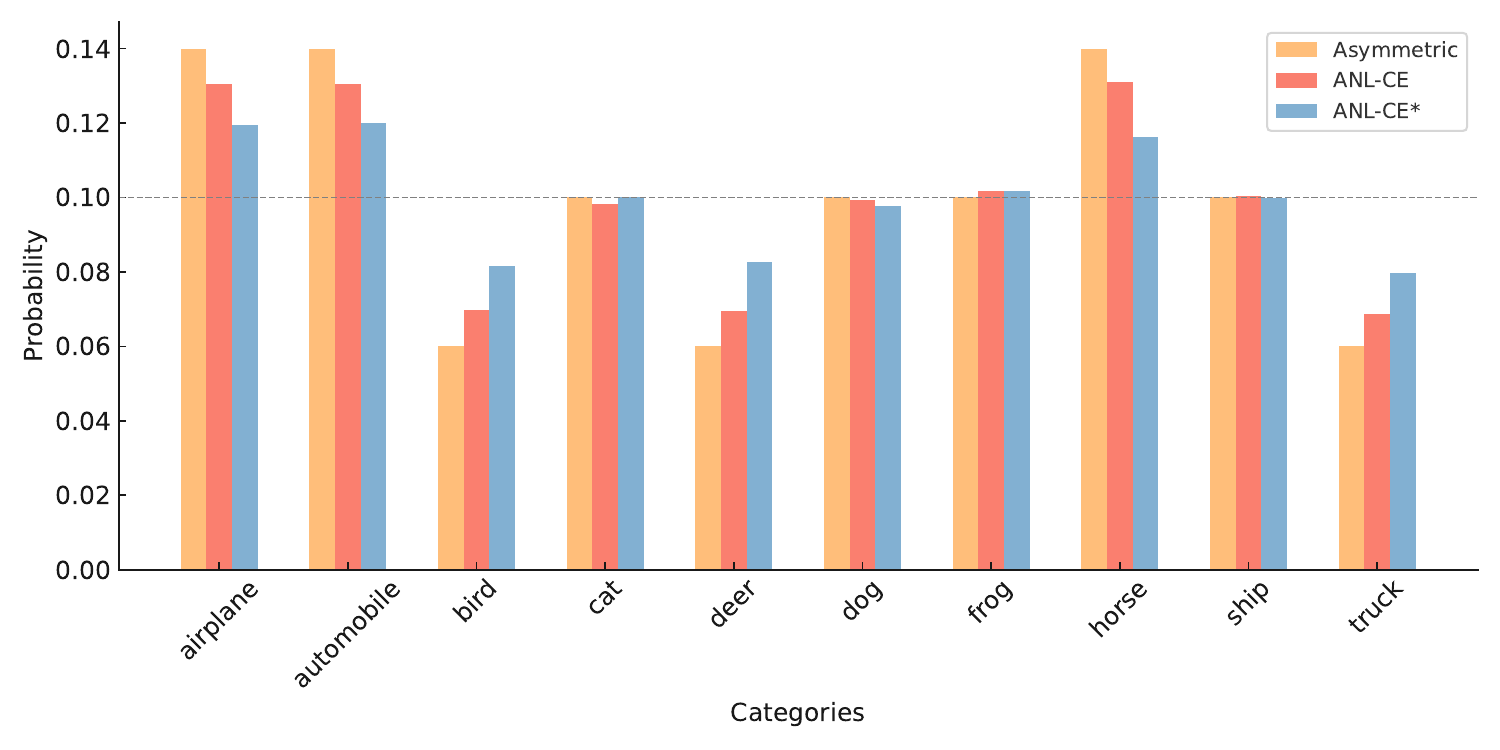}}
\caption{
    Marginal probability of each class on the CIFAR-10 training set with 40\% asymmetric label noise, along with model predictions using different approaches.
    ANL-CE/ANL-CE* is introduced in \Cref{eq:anl-ce}/\Cref{eq:anl-ce-star}, respectively.
}
\label{unbalanced label}
\end{center}
\vspace{-3mm}
\end{figure}

Although our proposed ANL framework has demonstrated outstanding performance in most label noise scenarios, it still faces challenges from asymmetric label noise at extremely high noise rates.
We attribute this issue to \textit{label imbalance} caused by the non-symmetric label noise.

To better understand this issue, we calculate the marginal probability of each class on CIFAR-10 training set under 40\% asymmetric label noise and compare it to the predictions of the model trained with ANL-CE, as shown in Figure~\ref{unbalanced label}.
It can be observed that the label distribution in the training data, corrupted by asymmetric label noise, is imbalanced.
Furthermore, our proposed ANL framework is vulnerable to this imbalance.

To this end, inspired by previous works \cite{li2021contrastive}, we introduce the following entropy regularization to mitigate the issue of imbalanced label:
\begin{equation}
    \mathcal{L}_\text{reg}(\bm \pi) = \log K + \sum_{k=1}^K \bm{\pi}(k) \log \bm{\pi}(k),
\end{equation}
where \(\bm{\pi}(k)\) represents the average probability that the model categorizes a sample as class \(k\) and is given by \(\bm{\pi}(k) = \frac{1}{N} \sum_{n=1}^N \bm{p}(k|\bm{x}_n)\), with \(N\) being the total number of samples.
In practice, we estimate \(\bm{\pi}(k)\) for each training batch as \(\bm{\pi}(k) = \frac{1}{B} \sum_{n=1}^B \bm{p}(k|\bm{x}_n)\), where \(B\) represents the batch size.
Following the principle of maximum entropy, this regularization term encourages the marginal probability for each class to be as uniform as possible, and particularly \(\mathcal{L}_\text{reg} = 0\) when \(\bm{\pi}(k) = \frac{1}{K}\) for all \(k\).

Finally, we integrate this regularization term into our ANL framework as follows:
\begin{equation}
\label{eq:anl-ce-star}
    \mathcal{L}_{\text{ANL*}}= \alpha  \mathcal{L}_\text{norm} + \beta  \mathcal{L}_\text{nn} + \lambda  \mathcal{L}_\text{reg},
\end{equation}
where \(\lambda\) is the weighting factor for the regularization term, and along this paper * indicates the ANL loss enhanced with this entropy regularizer.

\section{Theoretical Justification}

In this section, we theoretically investigate our proposed loss functions. First, we introduce the general properties of Active Negative Loss (ANL) functions, including their symmetric nature and noise tolerance, in Section~\ref{sec: properties_ANL}. 
Then in Section~\ref{sec: analysis_NNLFs} we provide a theoretical exploration and analysis of the gradients in using our Normalized Negative Loss Functions (compared to MAE).
The detailed proofs for theorems and corollaries can be found in the technical appendix, available in the online supplemental material.

\subsection{Proporties of Active Negative Loss}
\label{sec: properties_ANL}

In the following, we will prove that our proposed Normalized Negative Loss Functions (NNLFs) and Active Negative Loss (ANL) functions are robust to both \emph{symmetric} and \emph{asymmetric} label noise, theoretically enjoying the same guarantees as in \cite{ghosh2017robust, ma2020normalized}. 

\subsubsection{NNLFs are symmetric}

To ease the statement in the following proofs of noise robustness, we first prove that our proposed Normalized Negative Loss Functions (NNLFs) are symmetric.
We first review the definition of symmetric loss functions.
\begin{definition}[Symmetric loss function]
\label{def:sym}
A loss function $\mathcal{L}$ is called symmetric if it satisfies, for some constant $C$, $\sum_{k=1}^K \mathcal{L}(f(\boldsymbol{x}),k)=C$, $\forall \boldsymbol{x} \in \mathcal{X}$, $\forall f$.
\end{definition}
We then have a guarantee:
\begin{theorem}
    \label{thm:symmetric}
    Normalized negative loss function $\mathcal{L}_\text{nn}$, as specified in \Cref{eq:nnlf}, is symmetric.
\end{theorem}
Detailed proofs can be found in Appendix~\ref{appendix: noise tolerant}.
The general properties of symmetric loss functions have been well studied, and we refer reader to \cite{charoenphakdee2019symmetric} for more discussions on topics such as classification calibration and excess risk bound.

\subsubsection{NNLFs are robust to label noise}

For any loss function $\mathcal{L}$, given a classifier $f$, the $\mathcal{L}$-risk is defined as $R_\mathcal{L}(f) = \mathbb{E}_{\boldsymbol{x},y}[\mathcal{L}(f(\boldsymbol{x}),y)]$, where $\mathbb{E}$ is the expectation operator and its subscript denotes on which the expectation is taken.
Under risk minimization framework, the objective is to learn a global minimum of $R_\mathcal{L}(f)$, which is an optimal classifier $f^*(\cdot)$.
Similarly, the $\mathcal{L}$-risk under noisy data is defined as $R_\mathcal{L}^\eta(f) = \mathbb{E}_{\boldsymbol{x},\hat{y}}[\mathcal{L}(f(\boldsymbol{x}),\hat{y})]$, and we use $f^*_\eta(\cdot)$ to denote the global optimal classifier of $R_\mathcal{L}^\eta(f)$.
We then introduce the definition for being noise tolerant.
\begin{definition}[Noise tolerance~\cite{ghosh2017robust}.]
\label{def:tolerance}
A loss function $\mathcal{L}$ is \emph{noise tolerant}, if the global minimizer of $R_{\mathcal{L}}^{\eta}(f)$, $f^*_{\eta}$, is also the global minimum of $R_\mathcal{L}(f)$, $f^*$.
\end{definition}
By the definition above, for a noise tolerant loss function $\mathcal{L}$, its $f^*_\eta$ and $f^*$ will have the same misclassification probability on clean data, 
which reflects that the loss function $\mathcal{L}$ is robust to noise.
We show $\mathcal{L}_{nn}$ is noise tolerant as follows (we recall the definitions of different noise rates can be found in \Cref{sec: label_noise_model}):
\begin{theorem}
    \label{thm:sym-noise-robust}
    In a multi-class classification problem, any normalized negative loss function $\mathcal{L}_{nn}$ is noise tolerant under symmetric label noise, if noise rate $\eta< \frac{K-1}{K}$.
\end{theorem}
\begin{theorem}
    \label{thm:asym-noise-robust}
    In a multi-class classification problem, suppose $\mathcal{L}_{nn}$ satisfies $0 \le \mathcal{L}_{nn}(f(\boldsymbol{x}),k) \le 1$, $\forall k \in \{1,\cdots,K\}$.
    If $R(f^*)=0$, then, any normalized negative loss function $\mathcal{L}_{nn}$ is noise tolerant under asymmetric label noise when noise rate $\eta_{yk}<1-\eta_y$.
\end{theorem}
Detailed proofs are provided in Appendix~\ref{appendix: noise tolerant}.
\Cref{thm:sym-noise-robust,,thm:asym-noise-robust} show that under some mild assumptions, our proposed NNLFs are noise tolerant, which means they are robust to noise.

\subsubsection{ANL is robust to noise}

Similar to APL, we can show that our ANL is also robust to label noise. 
\begin{theorem}
    \label{thm:anl-noise-robust}
    $\forall \alpha, \beta$, the combination $\mathcal{L}_\text{ANL} = \alpha  \mathcal{L}_\text{norm} + \beta  \mathcal{L}_\text{nn}$ is noise tolerant under both symmetric and asymmetric label noise.
\end{theorem}
The proofs are outlined in Appendix~\ref{appendix: noise tolerant}.
Theorem~\ref{thm:anl-noise-robust} shows that the combination of normalized loss functions and our NNLFs are robust to noise.

\subsection{Gradient Analysis of NNLFs}
\label{sec: analysis_NNLFs}


As shown in Figure~\ref{fig: overfitting}, by replacing MAE with our proposed NNCE, ANL-CE
show better performance.
This raises the question: \emph{why does NNLF outperform MAE?}
To provide insight into this, we take NNCE as an example and analyze the gradients of MAE and NNCE.
Detailed derivations and proofs can be found in Appendix~\ref{appendix: gradient}.

The gradient of the MAE with respect to the classifier's output probability can be derived as:
\begin{equation}
\label{eq:grad-mae}
    \frac{\partial \mathcal{L}_\text{MAE}}{\partial \boldsymbol{p}(j|\boldsymbol{x})} =
    \begin{cases}
    \phantom{-}1, & j \ne y, \\
    -1, & j = y.
    \end{cases}
\end{equation}
The gradient of the NNCE with respect to the classifier's output probability can be derived as:
\begin{equation}
    \frac{\partial \mathcal{L}_\text{NNCE}}{\partial \boldsymbol{p}(j|\boldsymbol{x})} =
    \begin{cases}
    \phantom{-} \frac{1}{\boldsymbol{p}(j|\boldsymbol{x})} \frac{A + \log{\boldsymbol{p}(y|\boldsymbol{x})}}{\left( \sum_{k=1}^K \left( A + \log{\boldsymbol{p}(k|\boldsymbol{x})} \right) \right)^2}, & j \ne y, \\
    - \frac{1}{\boldsymbol{p}(y|\boldsymbol{x})} \frac{\sum_{k \ne y} \left( A + \log{\boldsymbol{p}(k|\boldsymbol{x})} \right)}{\left( \sum_{k=1}^K \left( A + \log{\boldsymbol{p}(k|\boldsymbol{x})} \right) \right)^2}, & j = y.
    \end{cases}
\end{equation}
For the sake of analysis, we consider how the gradients of NNCE (for example) would differ from MAE in the following two cases: 
a) given the classifier's output probability of sample $\boldsymbol{x}$, what is the difference in gradient for each class, 
b) given the classifier's output probabilities of sample $\boldsymbol{x}_1$ and $\boldsymbol{x}_2$, what is the difference in gradient between these two samples.
\begin{theorem}
    \label{thm:gradient-class}
    Given the classifier's output probability $\boldsymbol{p}(\cdot|\boldsymbol{x})$ with respect to sample $\boldsymbol{x}$ and normalized negative cross entropy $\mathcal{L}_\text{NNCE}$.
    If $\boldsymbol{p}(j_1|\boldsymbol{x})$ $<$ $\boldsymbol{p}(j_2|\boldsymbol{x})$, $j_1 \ne j_2 \ne y$, then \[\frac{\partial \mathcal{L}_\text{NNCE}}{\partial \boldsymbol{p}(j_1|\boldsymbol{x})} > \frac{\partial \mathcal{L}_\text{NNCE}}{\partial \boldsymbol{p}(j_2|\boldsymbol{x})}.\]
\end{theorem}
\begin{theorem}
    \label{thm:gradient-sample}
    Given the classifier's output probabilities $\boldsymbol{p}(\cdot|\boldsymbol{x}_1)$ and $\boldsymbol{p}(\cdot |\boldsymbol{x}_2)$ of sample $\boldsymbol{x}_1$ and $\boldsymbol{x}_2$, 
    where $\boldsymbol{p}(y|\boldsymbol{x}_1) = \max_{k \in [K]} \boldsymbol{p}(k|\boldsymbol{x}_1)$, $\boldsymbol{p}(y|\boldsymbol{x}_2) = \max_{k\in [K]} \boldsymbol{p}(k|\boldsymbol{x}_2)$. 
    Moreover, for better comparison, we assume there also exists a certain benchmarking class $j \ne y$ so that
    $\boldsymbol{p}(j|\boldsymbol{x}_1) = \boldsymbol{p}(j|\boldsymbol{x}_2)$. 
    
    If $\boldsymbol{p}(y|\boldsymbol{x}_1)$ $>$ $\boldsymbol{p}(y|\boldsymbol{x}_2)$ and $\boldsymbol{p}(k|\boldsymbol{x}_1) \le \boldsymbol{p}(k|\boldsymbol{x}_2)$, $\forall k \in [K] \setminus \{j,y\}$, then for normalized negative cross entropy $\mathcal{L}_\text{NNCE}$:
    \[\frac{\partial \mathcal{L}_\text{NNCE}}{\partial \boldsymbol{p}(j|\boldsymbol{x}_1)} > \frac{\partial \mathcal{L}_\text{NNCE}}{\partial \boldsymbol{p}(j|\boldsymbol{x}_2)}.\]
\end{theorem}
\Cref{thm:gradient-class,,thm:gradient-sample} show that, for the gradient of samples in non-labeled classes, MAE treats every class and sample equally (\cf~\Cref{eq:grad-mae});
however, our NNCE focuses more on the classes and the samples that have been well learned and this property tends to help the model perform better in noisy label learning.
Some studies \cite{li2020gradient, DBLP:conf/nips/LiuNRF20:ELR} have shown that during the training process, DNNs would first memorize clean samples and then noisy samples;
with the property revealed by \Cref{thm:gradient-class,,thm:gradient-sample}, apart from robustness, our NNLFs potentially help the model to continuously learn the clean samples that the model has memorized in the previous stages of training and ignore the unmemorized noisy samples.

\section{Experiments}
\label{experiments}

In this section, we evaluate our proposed ANL framework on both benchmark datasets, which include symmetric, asymmetric, instance-dependent, and real-world label noise, as well as large-scale real-world datasets.
Additionally, we explore its application in image segmentation.

\subsection{Datasets}

{\bf{CIFAR-10}} \cite{krizhevsky2009learning} dataset consists of 60,000 images of 10 categories.
We train the model on the 50,000 noisy training samples and evaluate it on 10,000 testing samples.

{\bf{CIFAR-100}} \cite{krizhevsky2009learning} including 100 classes belonging to 20 super-classes where each category contains 600 images.
Similarly, the 50,000 noisy training samples are used to train the model, while the remaining 10,000 testing samples are used to evaluate.

{\bf{WebVision}} \cite{wen2017webvision} dataset contains more than 2.4 million web images crawled from the internet by using queries generated from the 1,000 class labels of the ILSVRC 2012 \cite{imagenet} benchmark.
Here, we follow the ``Mini'' setting in \cite{jiang2018mentornet}, and only take the first 50 classes of the Google resized image subset.
We evaluate the trained networks on the same 50 classes of both the ILSVRC 2012 validation set and the WebVision validation set, and these can be considered as clean validation sets.

{\bf{Animal-10N}} \cite{song2019selfie} is a real-world noisy data set of human-labeled online images for 10 confusing animals, with $50,000$ training and $5,000$ testing images, and its noise rate was estimated at 8\%.

{\bf{Clothing-1M}} \cite{xiao2015learning} is a large-scale clothing dataset that is collected from real-world online shopping websites.
It contains 14 categories and 1 million training samples with nearly 40\% mislabeled samples.
This dataset also includes a set of clean samples, which has been divided into a training set ($50k$ samples), a validation set ($14k$ samples), and a test set ($10k$ samples).

{\bf{ISIC-2017}} \cite{isic} is a large-scale public medical benchmark dataset of dermoscopy images, curated for skin cancer detection. 
The dataset includes detailed annotations, such as lesion type, diagnosis, anatomical location, and corresponding boundary masks. 
Additionally, it provides lesion attributes like size, shape, and color, enabling more comprehensive analysis.

{\bf{Noise generation for image classification.}}
We briefly overview the noise generation process used for the CIFAR-10 and CIFAR-100 datasets.
For class-dependent noise, the noisy labels are generated following standard approaches in previous works \cite{patrini2017making, ma2020normalized, zhou2021asymmetric}.
For symmetric noise, we flip the labels in each class randomly to incorrect labels of other classes.
For asymmetric noise, we flip the labels within a specific set of classes.
For CIFAR-10, flipping TRUCK $\to$ AUTOMOBILE, BIRD $\to$ AIRPLANE, DEER $\to$ HORSE, CAT $\leftrightarrow$ DOG.
For CIFAR-100, the 100 classes are grouped into 20 super-classes with each has 5 sub-classes, and each class are flipped within the same super-class into the next in a circular fashion.
We vary the noise rate $\eta \in \{0.2, 0.4, 0.6, 0.8\}$ for symmetric noise and $\eta \in \{0.2,0.4\}$ for asymmetric noise.
For instance-dependent noise, we use the part-dependent noise from PDN \cite{DBLP:conf/nips/XiaL0WGL0TS20:PDN} with noise rate $\eta \in \{0.2, 0.4\}$, where the noise is synthesized based on the DNN prediction error.
For real-world noise, we use the ``Worst'' label set of CIFAR-10N and the ``Noisy'' label set of CIFAR-100N \cite{DBLP:conf/iclr/WeiZ0L0022:CIFAR-10N/-100N}, respectively.
The noise in CIFAR-10N and CIFAR-100N is generated based on human annotation errors.

{\bf{Noise generation for image segmentation.}}
To evaluate the model's robustness against low-quality annotations, we follow the methodology from \cite{DBLP:conf/miccai/LiGH21}. In this approach, a portion of the training data is randomly sampled with probability $\eta_\alpha \in \{0.3, 0.5, 0.7\}$, and morphological transformations are applied, with noise levels controlled by $\eta_\beta \in \{0.5, 0.7\}$. These morphological transformations include erosion, dilation, and affine transformations, which reduce, enlarge, or displace the annotated areas. This process effectively simulates the risk of annotation errors caused by factors such as annotator fatigue or the inherent difficulty of labeling certain images.

\subsection{Experimental Setup}
For CIFAR-10 and CIFAR-100 datasets, we follow the experimental setting in previous work \cite{ma2020normalized, zhou2021asymmetric}.
An 8-layer CNN is used for CIFAR-10, and a ResNet-34 \cite{he2016deep} is used for CIFAR-100.
For CIFAR-10 and CIFAR-100, the networks are trained for 120 and 200 epochs, respectively.
For all the training, we use SGD optimizer with momentum 0.9 and cosine learning rate annealing.
Weight decay is set to $1 \times 10^{-4}$, and $1 \times 10^{-5}$ for  CIFAR-10 and CIFAR-100, respectively.
Particularly, for our proposed ANL methods, weight decay is set to $0$ for all datasets.
The initial learning rate is set to $0.01$ for CIFAR-10 and $0.1$ for CIFAR-100.
Batch size is set to $128$.
For all settings, we clip the gradient norm to $5.0$.
Typical data augmentations including random width/height shift and horizontal flip are applied.

For WebVision dataset, we follow the experimental setting in previous works \cite{zhou2021asymmetric}.
We train a ResNet-50 \cite{he2016deep} using SGD for 250 epochs with initial learning rate 0.4, nesterov momentum 0.9 and weight decay $3 \times 10^{-5}$ (for our proposed ANL methods, the weight decay is set to $0$) and batch size 512.
The learning rate is multiplied by 0.97 after every epoch of training. For all settings, we clip the gradient norm to $5.0$.
All the images are resized to $224 \times 224$.
Typical data augmentations including random width/height shift, color jittering and random horizontal flip are applied.

For Animal-10N dataset, we follow the experimental setting in previous works \cite{song2019selfie}.
We use VGG-19 with batch normalization.
The SGD optimizer is employed.
We train the network for 100 epochs and use an initial learning rate of 0.1, which is divided by 5 at 50\% and 75\% of the total number of epochs.
Batch size is set to 128.
Typical data augmentations including random horizontal flip are applied.

For Clothing-1M dataset, we follow the experimental setting in previous works \cite{tanaka2018joint}.
We use the $14k$ and $10k$ clean data for validation and test, respectively, and we do not use the $50k$ clean training data.
We use ResNet-50 \cite{he2016deep} pre-trained on ImageNet.
For preprocessing, we resized the images to $256 \times 256$, performed mean subtraction, and cropped the middle $224 \times 224$.
We use the SGD optimizer with a momentum of 0.9, a weight decay of $1 \times 10^{-3}$, and batch size of 32.
We train the network for 10 epochs with learning rate $1 \times 10^{-3}$ and  $1 \times 10^{-4}$ for 5 epochs each.
Typical data augmentations including random horizontal flip are applied.

For ISIC-2017, we follow the experimental setting in previous work~\cite{DBLP:conf/miccai/LiGH21, DBLP:conf/miccai/medicalNoiseSegmentation}. 
For preprocessing, we use 2000 training images and 600 test images, all resized to $256 \times 256$ and augmented with random mirroring, flipping, and gamma transformations, as described in \cite{DBLP:conf/miccai/medicalNoiseSegmentation}. 
We use nnU-Net~\cite{nnunet} as a segmentation network.
We train it for 100 epochs, using the Adam \cite{Kingma2014adam} optimizer with an initial learning rate of $1 \times 10^{-3}$, which decays by a factor of $0.8$ after each epoch.
We set the weight decay to $5 \times 10^{-6}$ and use a batch size of 16.

\subsection{Empirical Understandings}
\label{experiment: empirical}
In this subsection, we investigate some properties of our proposed loss functions.
If not otherwise specified, all detailed experimental settings are the same as in Section~\ref{experiment: benchmarks}. 

\begin{figure*}
\centering
\captionsetup[subfigure]{labelformat=parens, labelsep=space}
\subfloat[CE]{\includegraphics[width=0.2\textwidth]{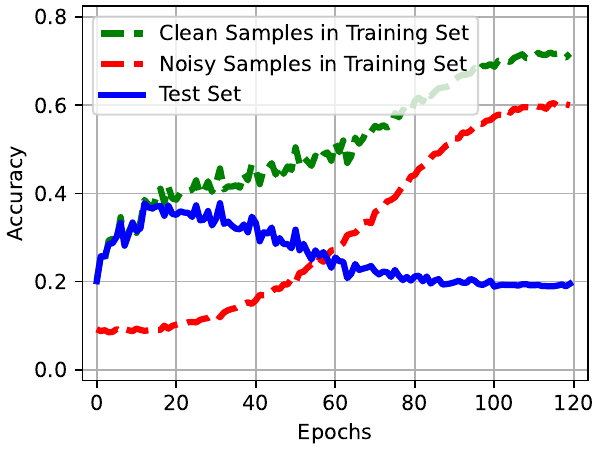}}
\subfloat[MAE]{\includegraphics[width=0.2\textwidth]{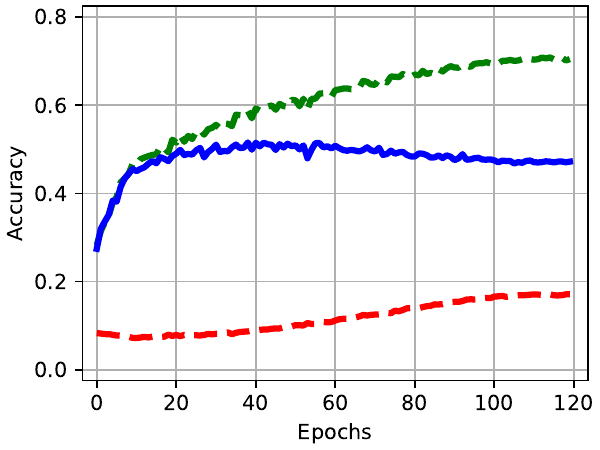}}
\subfloat[NCE+RCE]{\includegraphics[width=0.2\textwidth]{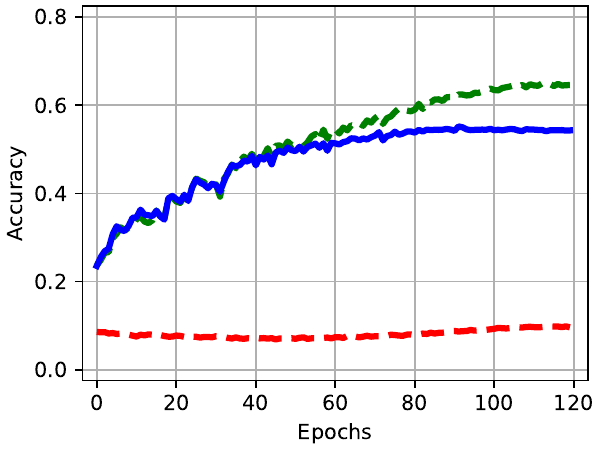}}
\subfloat[ANL-CE (w/ L2)]{\includegraphics[width=0.2\textwidth]{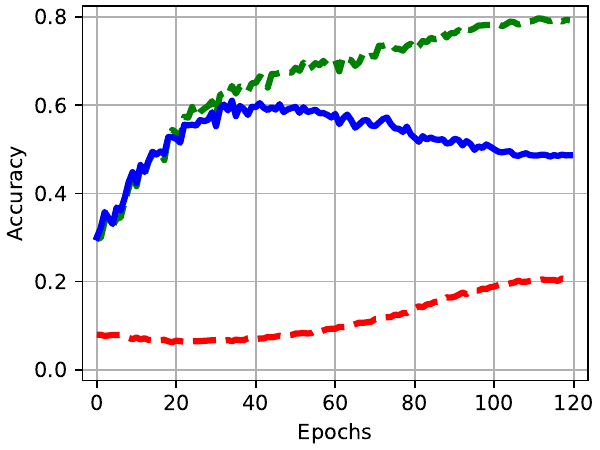}}
\subfloat[ANL-CE (w/ L1)]{\includegraphics[width=0.2\textwidth]{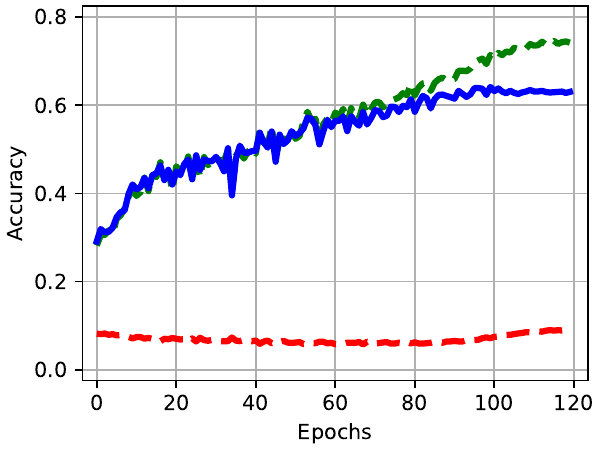}}

\subfloat[CE]{\includegraphics[width=0.2\textwidth]{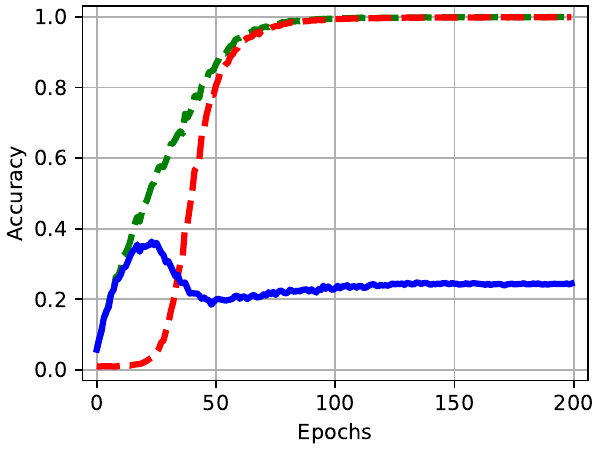}}
\subfloat[MAE]{\includegraphics[width=0.2\textwidth]{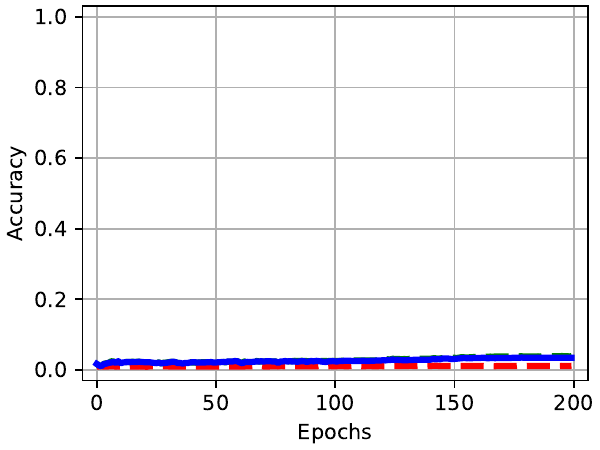}}
\subfloat[NCE+RCE]{\includegraphics[width=0.2\textwidth]{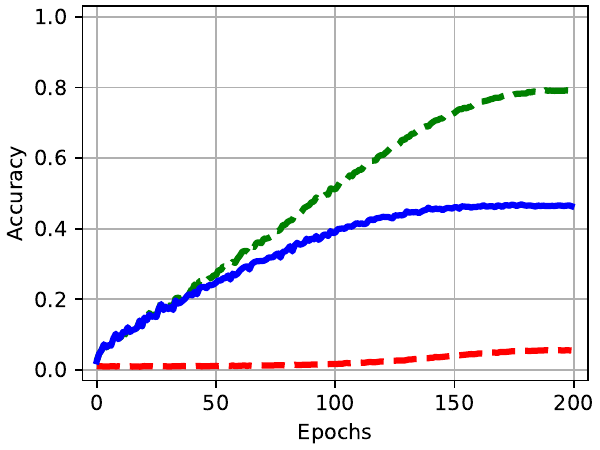}}
\subfloat[ANL-CE (w/ L2)]{\includegraphics[width=0.2\textwidth]{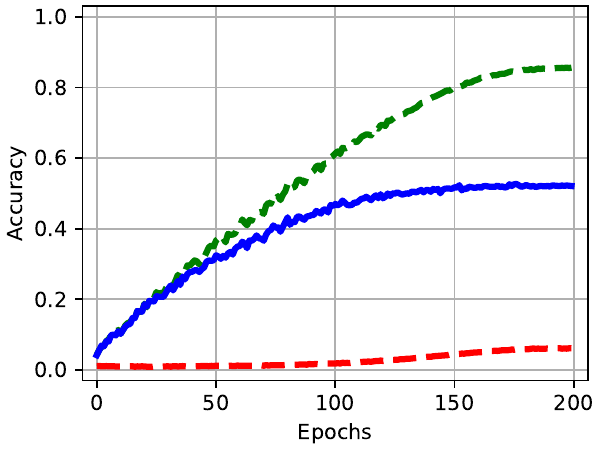}}
\subfloat[ANL-CE (w/ L1)]{\includegraphics[width=0.2\textwidth]{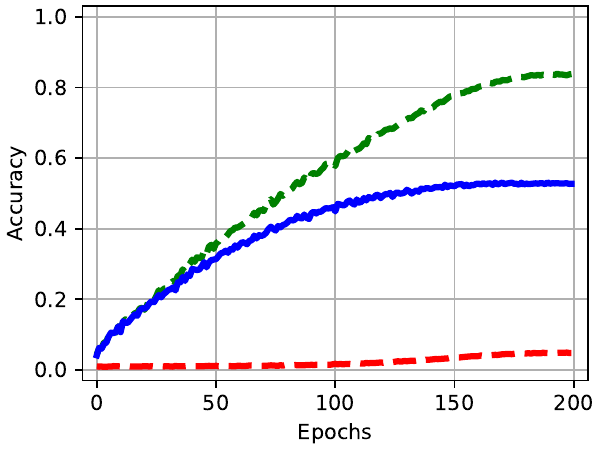}}
\caption{Training and test accuracies of different loss functions. (a) - (e): CIFAR-10 under 0.8 symmetric noise. (f) - (j): CIFAR-100 under 0.6 symmetric noise. The accuracies of noisy samples in training set should be as low as possible, since they are mislabeled.}
\label{fig: overfitting}
\vspace{-3mm}
\end{figure*}

\subsubsection{Overfitting problem}

In practice, we find that ANL can lead to overfitting in some experimental settings.
To motivate this problem, as an example, we train networks using different loss functions on CIFAR-10 under 0.8 symmetric noise and CIFAR-100 under 0.6 symmetric noise, and the experimental results are shown in Figure~\ref{fig: overfitting}.
As can be observed, in the setting of CIFAR-10 under 0.8 symmetric noise, the training set accuracy of ANL-CE (w/ L2) keeps increasing while the test set accuracy keeps decreasing.
We identify this problem as an \emph{overfitting problem}.

It is worth noting that although overfitting occurs, unlike CE, the gap between the clean sample accuracies and the noisy sample accuracies of the training set does not shrink, which indicates that our ANL has some robustness to noisy labels even in the case of overfitting.
Moreover, we conjecture that the overfitting is caused by the property that NNLF is more focus on well-learned samples.
When the noise rate is high, the model can be seen as trained on only a fairly small number of samples, and the lack of sample size leads to overfitting. 

\subsubsection{Robustness and fitting ability}

We conduct a set of experiments on CIFAR-10/-100 to verify the robustness and the fitting ability of our proposed loss functions.
We set the noise type to be symmetric and the noise rate to 0.8 for CIFAR-10 and 0.6 for CIFAR-100.
On each setting, we train network using different loss functions, including: a) CE, b) MAE, c) NCE+RCE, d) ANL-CE (w/ L2), and e) ANL-CE (w/ L1). For ANL-CE (w/ L2), we set its parameters $\alpha$ and $\beta$ to be the same as ANL-CE (w/ L1) and set its weight decay to be the same as NCE+RCE.

As can be observed in Figure~\ref{fig: overfitting}:
a) CE is not robust to noise, the accuracies of clean and noisy samples in the training set are continuously close to each other,
b) MAE is robust to noise, the accuracies of clean and noisy samples in the training set keep moving away from each other, but its fitting ability is insufficient, especially when the dataset becomes complex,
c) NCE+RCE is robust to noise and has better fitting ability compared to MAE,
d) ANL-CE (w/ L2) is robust to noise and has stronger fitting ability, but suffers from over-fitting.
and e) ANL-CE is robust to noise and mitigates the impact of overfitting to achieve the best performance.
To summarize, our proposed loss functions are robust to noise, NNLF shows better fitting ability than MAE, and L1 regularization addresses the overfitting problem of NNLF. 

\subsection{Experiments on CIFAR-10 and CIFAR-100}
\label{experiment: benchmarks}

\begin{table*}[t]
\begin{center}
\caption{
    Test accuracies (\%) of different methods on CIFAR-10 dataset with clean, symmetric ($\eta \in \{0.2, 0.4, 0.6, 0.8\}$), asymmetric ($\eta \in \{0.2, 0.4\}$), dependent (\(\eta \in \{0.4\}\)) and real-world label noise.
    The results (mean$\pm$std) are reported over 3 random runs under different random seeds ($1,2,3$).
    The top-1 results are in \textbf{bold}, while the second- and third-best results are \underline{underlined}.
}
\vspace{-1mm}
\label{tabel: cifar10}
\begin{tabular}{c|c|ccccccccc}
    \toprule
    Methods & Clean & Sym-0.2 & Sym-0.4 & Sym-0.6 & Sym-0.8 & Asym-0.2 & Asym-0.4 & Dep-0.2 & Dep-0.4 & Real \\
    \midrule
    CE & 90.38\tiny$\pm$0.11 & 75.05\tiny$\pm$0.26 & 58.19\tiny$\pm$0.21 & 38.75\tiny$\pm$0.19 & 19.09\tiny$\pm$0.35 & 83.00\tiny$\pm$0.33 & 73.69\tiny$\pm$0.20 & 75.51\tiny$\pm$0.30 & 56.36\tiny$\pm$0.60 & 61.43\tiny$\pm$0.52 \\
    FL \cite{lin2017focal} & 89.84\tiny$\pm$0.28 & 74.52\tiny$\pm$0.10 & 57.54\tiny$\pm$0.75 & 38.83\tiny$\pm$0.49 & 19.33\tiny$\pm$0.58 & 83.03\tiny$\pm$0.10 & 73.78\tiny$\pm$0.16 & 75.47\tiny$\pm$0.11 & 57.62\tiny$\pm$0.40 & 61.55\tiny$\pm$0.27 \\
    MAE \cite{ghosh2017robust} & 89.15\tiny$\pm$0.27 & 87.19\tiny$\pm$0.19 & 81.76\tiny$\pm$3.17 & 76.82\tiny$\pm$0.84 & 46.42\tiny$\pm$3.66 & 79.63\tiny$\pm$0.74 & 57.36\tiny$\pm$2.37 & 79.03\tiny$\pm$0.85 & 64.52\tiny\(\pm\)9.43 & 72.83\tiny$\pm$4.39 \\
    GCE \cite{zhang2018generalized} & 89.66\tiny$\pm$0.20 & 87.17\tiny$\pm$0.01 & 82.44\tiny$\pm$0.26 & 68.62\tiny$\pm$0.35 & 25.45\tiny$\pm$0.51 & 85.55\tiny$\pm$0.24 & 72.83\tiny$\pm$0.17 & 84.77\tiny$\pm$0.18 & 64.22\tiny$\pm$1.11 & 75.19\tiny$\pm$0.23 \\
    SCE \cite{wang2019symmetric} & 91.38\tiny$\pm$0.13 & 87.86\tiny$\pm$0.12 & 79.96\tiny$\pm$0.25 & 62.16\tiny$\pm$0.33 & 27.98\tiny$\pm$0.98 & 86.22\tiny$\pm$0.44 & 74.01\tiny$\pm$0.52 & 82.89\tiny$\pm$0.24 & 64.03\tiny$\pm$0.33 & 73.65\tiny$\pm$0.29 \\
    PHCE \cite{Menon2020clip} & 89.79\tiny$\pm$0.09 & 88.06\tiny$\pm$0.08 & 83.93\tiny$\pm$0.32 & 75.33\tiny$\pm$0.11 & 40.83\tiny$\pm$0.46 & 86.92\tiny$\pm$0.24 & 57.88\tiny$\pm$2.05 & 86.63\tiny$\pm$0.17 & 51.56\tiny$\pm$5.26 & 77.23\tiny$\pm$0.16 \\
    TCE \cite{Feng2020CanCE} & 90.01\tiny$\pm$0.17 & 87.57\tiny$\pm$0.14 & 83.82\tiny$\pm$0.13 & 75.32\tiny$\pm$0.14 & 38.10\tiny$\pm$0.42 & 86.84\tiny$\pm$0.09 & 59.08\tiny$\pm$2.34 & 86.51\tiny$\pm$0.27 & 59.86\tiny$\pm$3.21 & 77.32\tiny$\pm$0.56 \\
    APL \cite{ma2020normalized} & 90.94\tiny$\pm$0.01 & 89.19\tiny$\pm$0.18 &  86.03\tiny$\pm$0.13 & 79.89\tiny$\pm$0.11 & 55.52\tiny$\pm$2.74 & 88.36\tiny$\pm$0.13 & \underline{77.75\tiny$\pm$0.37} & 88.16\tiny$\pm$0.15 & 74.54\tiny$\pm$0.27 & 79.74\tiny$\pm$0.09 \\
    ALF \cite{zhou2021asymmetric} & 91.08\tiny$\pm$0.06 & 89.11\tiny$\pm$0.07 & 86.16\tiny$\pm$0.10 & 80.14\tiny$\pm$0.27 & 55.62\tiny$\pm$4.78 & 88.48\tiny$\pm$0.09 & \underline{78.60\tiny$\pm$0.41} & 88.14\tiny$\pm$0.40 & 74.40\tiny$\pm$1.04 & 79.91\tiny$\pm$0.37 \\
    \midrule
    \textbf{ANL-CE} & 91.66\tiny$\pm$0.04 & \textbf{90.02\tiny$\pm$0.23} & \textbf{87.28\tiny$\pm$0.02} & \underline{81.12\tiny$\pm$0.30} & \underline{61.27\tiny$\pm$0.55} & \underline{89.13\tiny$\pm$0.11} & 77.63\tiny$\pm$0.31 & \underline{88.70\tiny$\pm$0.08} & \underline{74.56\tiny$\pm$0.08} & \underline{80.23\tiny$\pm$0.28} \\
    \textbf{ANL-FL} & 91.79\tiny$\pm$0.19 & \underline{89.95\tiny$\pm$0.20} & \underline{87.25\tiny$\pm$0.11} & \textbf{81.67\tiny$\pm$0.19} & \underline{61.22\tiny$\pm$0.85} & \underline{89.09\tiny$\pm$0.31} & 77.73\tiny$\pm$0.31 & \underline{88.51\tiny$\pm$0.21} & \underline{74.91\tiny$\pm$0.09} & \underline{80.33\tiny$\pm$0.04} \\
    \textbf{ANL-CE*} & 91.52\tiny$\pm$0.06 & \underline{89.87\tiny$\pm$0.29} & \underline{86.66\tiny$\pm$0.15} & \underline{81.33\tiny$\pm$0.24} & \textbf{61.62\tiny$\pm$0.47} & \textbf{89.32\tiny$\pm$0.13} & \textbf{81.47\tiny$\pm$0.31} & \textbf{88.85\tiny$\pm$0.19} & \textbf{80.97\tiny$\pm$0.20} & \textbf{81.16\tiny$\pm$0.13} \\
    \bottomrule
\end{tabular}
\vspace{-3mm}
\end{center}
\end{table*}

\begin{table*}[t]
\begin{center}
\caption{
    Test accuracies (\%) of different methods on CIFAR-100 datasets with clean, symmetric ($\eta \in \{0.2, 0.4, 0.6, 0.8\}$), asymmetric ($\eta \in \{0.2, 0.4\}$), dependent (\(\eta \in \{0.2, 0.4\}\)) and real-world label noise.
    The results (mean$\pm$std) are reported over 3 random runs under different random seeds ($1,2,3$).
    The top-1 results are in \textbf{bold}, while the second- and third-best results are \underline{underlined}.
}
\vspace{-1mm}
\label{tabel: cifar100}
\begin{tabular}{c|c|ccccccccc}
    \toprule
    Methods & Clean & Sym-0.2 & Sym-0.4 & Sym-0.6 & Sym-0.8 & Asym-0.2 & Asym-0.4 & Dep-0.2 & Dep-0.4 & Real \\
    \midrule
    CE & 71.14\tiny$\pm$0.38 & 55.97\tiny$\pm$1.11 & 40.72\tiny$\pm$0.74 & 22.98\tiny$\pm$0.07 & ~7.55\tiny$\pm$0.21 & 58.25\tiny$\pm$1.00 & 41.53\tiny$\pm$0.34 & 57.47\tiny$\pm$0.44 & 43.41\tiny$\pm$0.32 & 48.63\tiny$\pm$0.53 \\
    FL \cite{lin2017focal} & 71.02\tiny$\pm$0.36 & 55.94\tiny$\pm$0.53 & 39.55\tiny$\pm$1.24 & 23.21\tiny$\pm$0.49 & ~7.80\tiny$\pm$0.27 & 58.00\tiny$\pm$1.38 & 41.88\tiny$\pm$0.57 & 58.37\tiny$\pm$0.41 & 42.97\tiny$\pm$0.54 & 49.09\tiny$\pm$0.68 \\
    MAE \cite{ghosh2017robust} & ~7.35\tiny$\pm$1.19 & ~7.91\tiny$\pm$0.66 & ~3.61\tiny$\pm$0.21 & ~3.63\tiny$\pm$0.35 & ~2.83\tiny$\pm$1.35 & ~6.19\tiny$\pm$0.42 & ~3.96\tiny$\pm$0.35 & ~5.80\tiny$\pm$2.41 & ~2.10\tiny$\pm$0.52 & ~3.24\tiny$\pm$1.35 \\
    GCE \cite{zhang2018generalized} & 61.62\tiny$\pm$0.43 & 61.50\tiny$\pm$1.50 & 56.46\tiny$\pm$0.95 & 46.27\tiny$\pm$1.30 & 19.51\tiny$\pm$0.86 & 59.06\tiny$\pm$0.46 & 41.51\tiny$\pm$0.52 & 60.68\tiny$\pm$0.86 & 50.30\tiny$\pm$0.86 & 50.97\tiny$\pm$0.60 \\
    SCE \cite{wang2019symmetric} & 70.80\tiny$\pm$0.37 & 55.04\tiny$\pm$0.37 & 39.84\tiny$\pm$0.19 & 21.97\tiny$\pm$0.92 & ~7.87\tiny$\pm$0.48 & 57.78\tiny$\pm$0.83 & 41.33\tiny$\pm$0.86 & 56.96\tiny$\pm$0.10 & 42.79\tiny$\pm$0.48 & 48.52\tiny$\pm$0.11 \\
    PHCE \cite{Menon2020clip} & 33.89\tiny$\pm$1.47 & 32.77\tiny$\pm$0.47 & 31.52\tiny$\pm$0.71 & 23.60\tiny$\pm$0.98 & 10.98\tiny$\pm$0.81 & 30.00\tiny$\pm$0.29 & 21.58\tiny$\pm$0.59 & 27.41\tiny$\pm$1.04 & 14.08\tiny$\pm$1.84 & 22.63\tiny$\pm$1.01 \\
    TCE \cite{Feng2020CanCE} & 38.14\tiny$\pm$3.25 & 37.08\tiny$\pm$0.71 & 38.62\tiny$\pm$1.41 & 31.48\tiny$\pm$1.38 & 15.48\tiny$\pm$0.63 & 34.78\tiny$\pm$0.82 & 26.84\tiny$\pm$1.24 & 34.80\tiny$\pm$2.52 & 18.84\tiny$\pm$0.37 & 24.18\tiny$\pm$0.83 \\
    APL \cite{ma2020normalized} & 68.22\tiny$\pm$0.28 & 64.20\tiny$\pm$0.47 & 57.97\tiny$\pm$0.30 & 46.26\tiny$\pm$1.07 & 25.65\tiny$\pm$0.51 & 62.77\tiny$\pm$0.53 & 42.46\tiny$\pm$0.42 & 63.84\tiny$\pm$0.29 & 50.37\tiny$\pm$0.37 & 54.27\tiny$\pm$0.09 \\
    ALF \cite{zhou2021asymmetric} & 68.61\tiny$\pm$0.12 & 65.30\tiny$\pm$0.21 & 59.74\tiny$\pm$0.68 & 47.96\tiny$\pm$0.44 & 24.13\tiny$\pm$0.07 & 64.05\tiny$\pm$0.25 & 44.90\tiny$\pm$0.62 & 64.60\tiny$\pm$0.51 & 51.47\tiny$\pm$0.59 & 55.96\tiny$\pm$0.20 \\
    \midrule
    \textbf{ANL-CE} & 70.68\tiny$\pm$0.23 & \underline{66.79\tiny$\pm$0.34} &  \textbf{61.80\tiny$\pm$0.50} & \underline{51.52\tiny$\pm$0.53} & \underline{28.07\tiny$\pm$0.28} & \textbf{66.27\tiny$\pm$0.19} & \underline{45.41\tiny$\pm$0.68} & \textbf{66.31\tiny$\pm$0.20} & \textbf{55.84\tiny$\pm$0.33} & \underline{56.37\tiny$\pm$0.42} \\
    \textbf{ANL-FL} & 70.40\tiny$\pm$0.15 & \underline{66.54\tiny$\pm$0.29} & \underline{61.73\tiny$\pm$0.48} & \underline{51.32\tiny$\pm$0.34} & \underline{27.97\tiny$\pm$0.58} & \underline{66.26\tiny$\pm$0.44} & \textbf{46.65\tiny$\pm$0.04} & \underline{66.02\tiny$\pm$0.29} & \underline{55.40\tiny$\pm$0.37} & \textbf{57.03\tiny$\pm$0.38} \\
    \textbf{ANL-CE*} & 70.27\tiny$\pm$0.24 & \textbf{67.23\tiny$\pm$0.14} & \underline{61.65\tiny$\pm$0.52} & \textbf{51.59\tiny$\pm$0.36} & \textbf{29.45\tiny$\pm$0.42} & \underline{65.71\tiny$\pm$0.17} & \underline{46.57\tiny$\pm$0.26} & \underline{66.15\tiny$\pm$0.47} & \underline{55.33\tiny$\pm$0.70} & \underline{56.85\tiny$\pm$0.39} \\
    \bottomrule
\end{tabular}
\vspace{-4mm}
\end{center}
\end{table*}

\textbf{Baselines.}
We consider several state-of-the-art methods:
a) Generalized Cross Entropy (GCE) \cite{zhang2018generalized},
b) Symmetric Cross Entropy (SCE) \cite{wang2019symmetric},
c) Partially Huberised Cross Entropy (PHCE) \cite{Menon2020clip},
d) Taylor Cross Entropy (TCE) \cite{Feng2020CanCE},
e) Active Passive Loss (APL) \cite{ma2020normalized}, specifically NCE+RCE,
f) Asymmetric Loss Functions (AFL) \cite{zhou2021asymmetric}, specifically NCE+AGCE.
For our proposed ANL framework, we consider three loss functions: a) ANL-CE, b) ANL-FL and c) ANL-CE* (ANL-CE with entropy regularization).
We also train networks using Cross Entropy (CE), Focal Loss (FL) \cite{lin2017focal}, and Mean Absolute Error (MAE) \cite{ghosh2017robust}.

\textbf{Results.}
The main experimental results under symmetric, asymmetric, instance-dependent, and real-world label noise for CIFAR-10 and CIFAR-100 are summarized in Table~\ref{tabel: cifar10} and Table~\ref{tabel: cifar100}, respectively.

Table~\ref{tabel: cifar10} displays the CIFAR-10 results, demonstrating that our proposed ANL-CE and ANL-FL methods consistently outperform baseline approaches across most noise types. 
The only exception occurs under 0.4 asymmetric noise, where they fall slightly behind APL and ALF. 
However, in all other cases, our methods exhibit superior performance. 
Specifically, under 0.8 symmetric noise, ANL-CE achieves a significant improvement, surpassing the state-of-the-art method by more than 5.0\%. 
Additionally, ANL-CE*, ANL-CE
with our entropy regularization
, exceeds all baseline methods across various noise types.
Moreover, ANL-CE* outperforms both ANL-CE and ANL-FL under non-symmetric noise with high noise rates, highlighting the effectiveness of our regularization approach in addressing the vulnerabilities of the original ANL framework to imbalanced label. 
For example, under 0.4 asymmetric noise, ANL-CE* outperforms the previous best methods (ALF) by 2.87\%.

Likewise, Table~\ref{tabel: cifar100} presents the CIFAR-100 results, where ANL-CE, ANL-FL, and ANL-CE* consistently outperform all baseline methods across different noise types. 
Under symmetric noise, as the noise rate increases, the performance improvements of our methods become even more significant. 
For example, under 0.2 and 0.4 symmetric noise, ANL-CE, ANL-FL, and ANL-CE* outperform other methods by less than 2\%, while at 0.6 the margin increases to over 3
\%, and at 0.8, it extends to over 5\%.
Notably, ANL-CE* also surpasses the original ANL-CE under
0.4 asymmetric noise, indicating that our entropy-based regularization technique enhances robustness against imbalanced labels.
Moreover, in scenarios with asymmetric, instance-dependent, and real-world noise, our methods continue to exhibit superior resilience, consistently outperforming the baseline models.

Overall, the results highlight the strong performance of our ANL framework across multiple datasets, noise types, and noise rates, validating the effectiveness and robustness of the proposed approach.

\begin{table}[t]
    \caption{
        Top-1 validation accuracies (\%) on clean ILSVRC12 and WebVision validation set of ResNet-50 models trained on WebVision using different methods. The top-1 results are in \textbf{bold}, and the second-best results are \underline{underlined}.}
    \vspace{-2mm}
    \label{tabel: webvision}
    \begin{center}
    \begin{tabular}{c|cc}
        \toprule
        Methods & ILSVRC12 Val & WebVision Val \\
        \midrule
        CE & 58.64 & 61.20 \\
        GCE \cite{zhang2018generalized} & 56.56 & 59.44 \\
        SCE \cite{wang2019symmetric} & 62.60 & {\underline{68.00}} \\
        APL \cite{ma2020normalized} & 62.40 & 64.92 \\
        ALF \cite{zhou2021asymmetric} & 60.76 & 63.92 \\
        \midrule
        ANL-CE & {\underline{65.00}} & 67.44 \\
        ANL-FL & {\bf{65.56}} & {\bf{68.32}} \\
        \bottomrule
    \end{tabular}
    \vspace{-3mm}
    \end{center}
\end{table}

\begin{table}[t]
    \caption{Test accuracies (\%) of different methods on Animal-10N dataset. The results (mean$\pm$std) are reported over 3 random runs under different random seeds ($1,2,3$) and the top-1 results are in \textbf{bold}.}
    \vspace{-2mm}
    \label{table: animal10n}
    \begin{center}
    \begin{tabular}{c|ccc}
    \toprule
    Methods & CE & GCE \cite{zhang2018generalized} & ANL-CE (ours) \\
    \midrule
    Test Acc. (\%) & 78.92\tiny$\pm$0.76 & 80.39\tiny$\pm$0.17 & \textbf{80.72\tiny$\pm$0.37} \\
    \bottomrule
    \end{tabular}
    \vspace{-4mm}
    \end{center}
\end{table}

\begin{table}[t]
    \caption{Test accuracies (\%) of different methods on Clothing-1M dataset. The top-1 results are in \textbf{bold}.}
    \vspace{-2mm}
    \label{table: clothing1m}
    \begin{center}
    \begin{tabular}{c|cccc}
    \toprule
    Methods & CE & GCE \cite{zhang2018generalized} & APL \cite{ma2020normalized} & ANL-CE (ours) \\
    \midrule
    Test Acc. (\%) & 68.07 & 68.94 & 69.07 & \textbf{69.93} \\
    \bottomrule
    \end{tabular}
    \vspace{-3mm}
    \end{center}
\end{table}

\begin{table*}[t]
    \caption{
        Dice score (\%) of different methods on ISIC-2017 with different noise rates. The results (mean$\pm$std) are reported from 3 random runs under different random seeds over the last 10 epochs. The top-1 results are in \textbf{bold}.}
    \vspace{-2mm}
    \label{table: isic}
    \begin{center}
        \begin{tabular}{c|c|cccccc}
        \toprule
        \multirow{2}{*}{Methods} & \multirow{2}{*}{$\eta_\alpha=0.0$} & \multicolumn{2}{c}{$\eta_\alpha=0.3$} & \multicolumn{2}{c}{$\eta_\alpha=0.5$} & \multicolumn{2}{c}{$\eta_\alpha=0.7$} \\
        & & $\eta_\beta=0.5$ & $\eta_\beta=0.7$ & $\eta_\beta=0.5$ & $\eta_\beta=0.7$ & $\eta_\beta=0.5$ & $\eta_\beta=0.7$ \\ \midrule
        GCE \cite{zhang2018generalized}          & 82.8\tiny$\pm$0.7    & 80.5\tiny$\pm$0.9   & 78.5\tiny$\pm$1.4 & 77.2\tiny$\pm$1.5 & 73.6\tiny$\pm$1.7 & 74.3\tiny$\pm$1.2 & 69.1\tiny$\pm$2.2 \\
        MAE \cite{ghosh2017robust}          & 82.6\tiny$\pm$0.6    & 80.3\tiny$\pm$0.7   & 78.6\tiny$\pm$1.2 & 77.1\tiny$\pm$1.0 & 74.2\tiny$\pm$1.4 & 75.1\tiny$\pm$0.9 & 69.8\tiny$\pm$2.0 \\
        SCE \cite{wang2019symmetric}          & 82.8\tiny$\pm$0.6    & 80.6\tiny$\pm$0.9   & 79.3\tiny$\pm$1.1 & 77.4\tiny$\pm$1.2 & 73.8\tiny$\pm$1.4 & 75.6\tiny$\pm$1.0 & 69.1\tiny$\pm$2.0 \\
        APL \cite{ma2020normalized} & 82.9\tiny$\pm$0.6 & 79.9\tiny$\pm$0.8   & 79.2\tiny$\pm$1.2 & 77.7\tiny$\pm$1.1 & 75.1\tiny$\pm$1.8 & 74.6\tiny$\pm$1.1 & 69.6\tiny$\pm$1.3 \\
        T-Loss \cite{DBLP:conf/miccai/medicalNoiseSegmentation} & 82.5\tiny$\pm$0.5 & 80.9\tiny$\pm$0.6   & 80.4\tiny$\pm$0.5 & 80.0\tiny$\pm$1.1 & \textbf{79.0\tiny$\pm$0.5} & 78.8\tiny$\pm$0.7 & 76.1\tiny$\pm$0.6 \\ \midrule
        ANL-CE & 82.6\tiny$\pm$0.1 & \textbf{81.7\tiny$\pm$0.1} & \textbf{80.8\tiny$\pm$0.5} & \textbf{80.4\tiny$\pm$0.3} & 78.7\tiny$\pm$0.5 & \textbf{79.6\tiny$\pm$0.6} & \textbf{77.3\tiny$\pm$0.4} \\
        \bottomrule
    \end{tabular}
    \vspace{-3mm}
    \end{center}
\end{table*}

\subsection{Experiments on Large-scale Real-world Datasets}

In this section, we evaluate the performance of our proposed ANL-CE and ANL-FL methods on three large-scale, real-world noisy datasets: WebVision\cite{wen2017webvision}, Animal-10N\cite{song2019selfie}, Clothing-1M\cite{xiao2015learning}. We compare our methods against several state-of-the-art approaches. The detailed results are presented in Table~\ref{tabel: webvision}, Table~\ref{table: animal10n}, and Table~\ref{table: clothing1m}. 

\subsubsection{WebVision Results}
As illustrated in Table~\ref{tabel: webvision}, we compare our proposed ANL-CE and ANL-FL methods against several state-of-the-art approaches, including GCE \cite{zhang2018generalized}, SCE \cite{wang2019symmetric}, APL \cite{ma2020normalized}, and AFL \cite{zhou2021asymmetric}. For the ILSVRC12 validation set, ANL-FL achieves the best accuracy of 65.56\%, followed closely by ANL-CE with 65.00\%. For the WebVision validation set, ANL-FL also attains the highest accuracy with 68.32\%. Additionally, ANL-CE achieves 67.44\%, trailing the second-best method, SCE, by only 0.56\%, still exhibiting its competitive performance.

\subsubsection{Animal-10N Results}
Table~\ref{table: animal10n} shows the results on the Animal-10N dataset, where we compare ANL-CE with CE and GCE \cite{zhang2018generalized}. ANL-CE achieves the highest test accuracy at 80.72\%, surpassing both baseline methods. 

\subsubsection{Clothing-1M Results}
For the Clothing-1M dataset, we compare ANL-CE with CE, GCE \cite{zhang2018generalized}, and APL \cite{ma2020normalized}. As shown in Table~\ref{table: clothing1m}, ANL-CE outperforms the best baseline (APL), achieving a top accuracy of 69.93\%.

\subsection{Experiments on Image Segmentation}

\begin{figure}[t]
\begin{center}
\centerline{\includegraphics[width=0.48\textwidth]{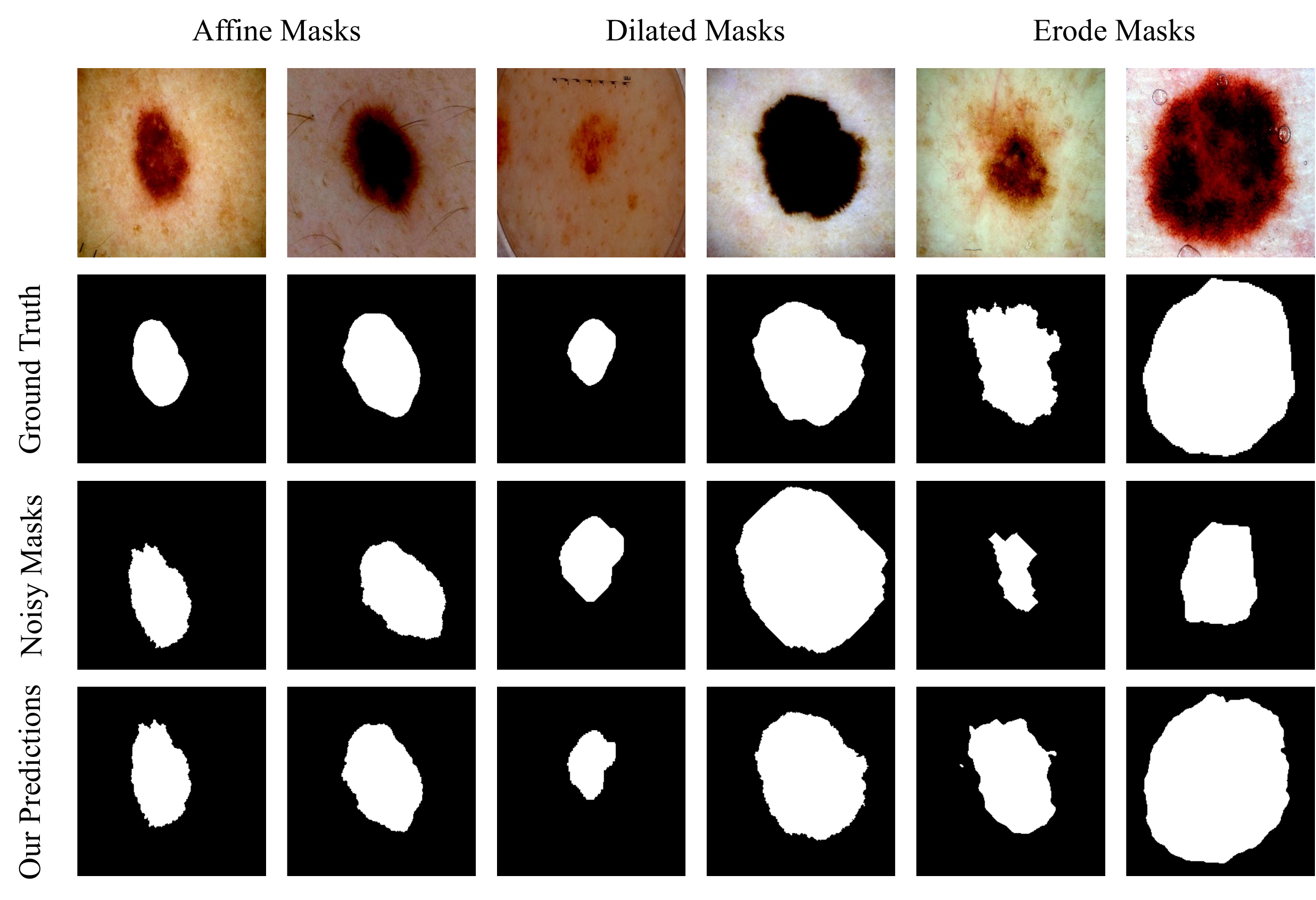}}
\vspace{-2mm}
\caption{
    Example segmentation results using ANL-CE on ISIC-2017 dataset under the noise scenario with $\eta_\alpha = 0.7$ and $\eta_\beta = 0.7$.
    The images show predicted masks under three noise conditions: affine masks, dilated masks, and eroded masks.
    Each row represents, from top to bottom: original images, ground truth masks, noisy masks, and our predicted masks.
}
\label{fig: segment result}
\end{center}
\vspace{-3mm}
\end{figure}

In this section, we evaluate the performance of our proposed ANL-CE methods on the skin lesion segmentation task to validate the noise tolerance of our proposed ANL when facing low-quality annotations in image segmentation.

\textbf{Baselines.}
We consider several state-of-the-art methods:
a) Generalized Cross Entropy (GCE) \cite{zhang2018generalized},
b) Mean Absolute Error \cite{ghosh2017robust},
c) Symmetric Cross Entropy (SCE) \cite{wang2019symmetric},
d) Active Passive Loss (APL) \cite{ma2020normalized}, specifically NCE+RCE,
e) T-Loss \cite{DBLP:conf/miccai/medicalNoiseSegmentation}.
For our proposed ANL framework, we consider loss functions ANL-CE.

\textbf{Results.}
We present experimental results for the skin lesion segmentation task on ISIC-2017 under various noise rates in Table~\ref{table: isic}. 
Our proposed ANL-CE consistently outperforms baseline methods across most noise types, with the only exception being when $\eta_\alpha = 0.7$ and $\eta_\beta = 0.7$, where it falls slightly behind T-Loss. 
Notably, our method achieves a dice score of 77.3\% even in the most extreme noise scenario, surpassing the previous best method by 0.8\%. 
Examples of the predicted masks under this extreme noise condition are shown in Figure~\ref{fig: segment result}. 
As illustrated in both Table~\ref{table: isic} and Figure~\ref{fig: segment result}, ANL-CE exhibits superior performance in high noise scenarios, further validating the robustness of our ANL framework against noisy labels in image segmentation tasks.

\section{Conclusion} \label{conclusion}
In this paper, we propose \emph{Active Negative Loss} (ANL), a framework of robust loss function with better performance.
A key component of this framework is \emph{Normalized Negative Loss Functions} (NNLFs), which integrates complementary label learning, ``vertical flipping'', and normalization. 
This allows us to convert any active loss function into a noise-robust passive one. 
We theoretically demonstrate that our NNLFs prioritize well-learned samples, and the ANL framework as a whole is robust to noisy labels.
Furthermore, considering the imbalance in label distribution caused by asymmetric label noise, 
we introduce an entropy-based regularization technique to further enhance its robustness. 
Along with this improvement, ANL consistently outperforms state-of-the-art methods across various types of label noise, including symmetric, asymmetric, instance-dependent, and real-world noise, as evidenced in our experiments. 
Moreover, we extend the ANL framework from the scenario of image classification to image segmentation, in which ANL keep surpassing current state-of-the-art approaches, showcasing the robustness and versatility of our method.

\section*{Acknowledgments}
This work is supported by the Shanghai Engineering Research Center of Intelligent Computing System (Grant No. 19DZ2252600), the Research Grants Council (RGC) under grant ECS-22303424, and the National Natural Science Foundation of China (Grant No. 62472097).
The authors also thank Prof.\ Cheng Jin for providing the computational resources that significantly contributed to the success of this research.

\bibliographystyle{IEEEtran}
\bibliography{IEEE}

@inproceedings{ma2020normalized,
  author    = {Xingjun Ma and
               Hanxun Huang and
               Yisen Wang and
               Simone Romano and
               Sarah M. Erfani and
               James Bailey},
  title     = {Normalized Loss Functions for Deep Learning with Noisy Labels},
  booktitle = {Proc. Int. Conf. Mach. Learn.},
  pages     = {6543--6553},
  year      = {2020}
}

@inproceedings{zhou2021asymmetric,
  author    = {Xiong Zhou and
               Xianming Liu and
               Junjun Jiang and
               Xin Gao and
               Xiangyang Ji},
  title     = {Asymmetric Loss Functions for Learning with Noisy Labels},
  booktitle = {Proc. Int. Conf. Mach. Learn.},
  pages     = {12846--12856},
  year      = {2021}
}

@techreport{krizhevsky2009learning,
  author = {Alex Krizhevsky and Geoffrey Hinton},
  title = {Learning Multiple Layers of Features from Tiny Images},
  year = {2009},
  institution = {University of Toronto},
}

@inproceedings{zhang2018generalized,
  author    = {Zhilu Zhang and
               Mert R. Sabuncu},
  title     = {Generalized Cross Entropy Loss for Training Deep Neural Networks with Noisy Labels},
  booktitle = {Adv. Neural Inf. Process. Syst.},
  pages     = {8792--8802},
  year      = {2018}
}

@inproceedings{zhang2017understanding,
  author    = {Chiyuan Zhang and
               Samy Bengio and
               Moritz Hardt and
               Benjamin Recht and
               Oriol Vinyals},
  title     = {Understanding Deep Learning Requires Rethinking Generalization},
  booktitle = {Proc. Int. Conf. Learn. Representations},
  year      = {2017}
}

@inproceedings{ghosh2017robust,
  author    = {Aritra Ghosh and Himanshu Kumar and P. S. Sastry},
  title     = {Robust Loss Functions under Label Noise for Deep Neural Networks},
  booktitle = {Proc. AAAI Conf. Artif. Intell.},
  pages     = {1919--1925},
  year      = {2017}
}

@inproceedings{wang2019symmetric,
  author    = {Yisen Wang and
               Xingjun Ma and
               Zaiyi Chen and
               Yuan Luo and
               Jinfeng Yi and
               James Bailey},
  title     = {Symmetric Cross Entropy for Robust Learning with Noisy Labels},
  booktitle = {Proc. IEEE Int. Conf. Comput. Vis.},
  pages     = {322--330},
  year      = {2019}
}

@inproceedings{lin2017focal,
  author    = {Tsung{-}Yi Lin and
               Priya Goyal and
               Ross B. Girshick and
               Kaiming He and
               Piotr Doll{\'{a}}r},
  title     = {Focal Loss for Dense Object Detection},
  booktitle = {Proc. IEEE Int. Conf. Comput. Vis.},
  pages     = {2999--3007},
  year      = {2017}
}

@inproceedings{natarajan2013learning,
  author    = {Nagarajan Natarajan and
               Inderjit S. Dhillon and
               Pradeep Ravikumar and
               Ambuj Tewari},
  title     = {Learning with Noisy Labels},
  booktitle = {Adv. Neural Inf. Process. Syst.},
  pages     = {1196--1204},
  year      = {2013}
}

@article{wen2017webvision,
  author    = {Wen Li and
               Limin Wang and
               Wei Li and
               Eirikur Agustsson and
               Luc Van Gool},
  title     = {WebVision Database: Visual Learning and Understanding from Web Data},
  journal   = {arXiv:1708.02862},
  year      = {2017}
}

@inproceedings{he2016deep,
  author    = {Kaiming He and
               Xiangyu Zhang and
               Shaoqing Ren and
               Jian Sun},
  title     = {Deep Residual Learning for Image Recognition},
  booktitle = {Proc. IEEE Conf. Comput. Vis. Pattern Recognit.},
  pages     = {770--778},
  year      = {2016}
}

@inproceedings{charoenphakdee2019symmetric,
  author    = {Nontawat Charoenphakdee and
               Jongyeong Lee and
               Masashi Sugiyama},
  title     = {On Symmetric Losses for Learning from Corrupted Labels},
  booktitle = {Proc. Int. Conf. Mach. Learn.},
  pages     = {961--970},
  year      = {2019}
}

@article{sokolic2017robust,
  author    = {Jure Sokolic and
               Raja Giryes and
               Guillermo Sapiro and
               Miguel R. D. Rodrigues},
  title     = {Robust Large Margin Deep Neural Networks},
  journal   = {{IEEE} Trans. Signal Process.},
  volume    = {65},
  number    = {16},
  pages     = {4265--4280},
  year      = {2017}
}

@article{hoffman2019robust,
  author    = {Judy Hoffman and
               Daniel A. Roberts and
               Sho Yaida},
  title     = {Robust Learning with Jacobian Regularization},
  journal   = {arXiv:1908.02729},
  year      = {2019},
}

@inproceedings{jiang2018mentornet,
  author    = {Lu Jiang and
               Zhengyuan Zhou and
               Thomas Leung and
               Li{-}Jia Li and
               Li Fei{-}Fei},
  title     = {MentorNet: Learning Data-Driven Curriculum for Very Deep Neural Networks on Corrupted Labels},
  booktitle = {Proc. Int. Conf. Mach. Learn.},
  pages     = {2309--2318},
  year      = {2018}
}

@inproceedings{han2018co,
  author    = {Bo Han and
               Quanming Yao and
               Xingrui Yu and
               Gang Niu and
               Miao Xu and
               Weihua Hu and
               Ivor W. Tsang and
               Masashi Sugiyama},
  title     = {Co-teaching: Robust Training of Deep Neural Networks with Extremely Noisy Labels},
  booktitle = {Adv. Neural Inf. Process. Syst.},
  pages     = {8536--8546},
  year      = {2018}
}

@inproceedings{patrini2017making,
  author    = {Giorgio Patrini and
               Alessandro Rozza and
               Aditya Krishna Menon and
               Richard Nock and
               Lizhen Qu},
  title     = {Making Deep Neural Networks Robust to Label Noise: A Loss Correction Approach},
  booktitle = {Proc. IEEE Conf. Comput. Vis. Pattern Recognit.},
  pages     = {2233--2241},
  year      = {2017}
}

@inproceedings{ishida2017learning,
  author    = {Takashi Ishida and
               Gang Niu and
               Weihua Hu and
               Masashi Sugiyama},
  title     = {Learning from Complementary Labels},
  booktitle = {Adv. Neural Inf. Process. Syst.},
  pages     = {5639--5649},
  year      = {2017}
}

@inproceedings{yu2018learning,
  author    = {Xiyu Yu and
               Tongliang Liu and
               Mingming Gong and
               Dacheng Tao},
  title     = {Learning with Biased Complementary Labels},
  booktitle = {Proc. Eur. Conf. Comput. Vis.},
  pages     = {69--85},
  year      = {2018}
}

@inproceedings{xiao2015learning,
  author    = {Tong Xiao and
               Tian Xia and
               Yi Yang and
               Chang Huang and
               Xiaogang Wang},
  title     = {Learning from Massive Noisy Labeled Data for Image Classification},
  booktitle = {Proc. IEEE Conf. Comput. Vis. Pattern Recognit.},
  pages     = {2691--2699},
  year      = {2015}
}

@inproceedings{Feng2020CanCE,
  title     = {Can Cross Entropy Loss Be Robust to Label Noise?},
  author    = {Lei Feng and Senlin Shu and Zhuoyi Lin and Fengmao Lv and Li Li and Bo An},
  booktitle = {Proc. Int. Joint Conf. Artif. Intell.},
  year      = {2020}
}

@inproceedings{Menon2020clip,
  title     = {Can Gradient Clipping Mitigate Label Noise?},
  author    = {Aditya Krishna Menon and Ankit Singh Rawat and Sanjiv Kumar and Sashank Reddi},
  booktitle = {Proc. Int. Conf. Learn. Representations},
  year      = {2020}
}

@inproceedings{Kingma2014adam,
  author    = {Diederik P. Kingma and Jimmy Ba},
  title     = {Adam: {A} Method for Stochastic Optimization},
  booktitle = {Proc. Int. Conf. Learn. Representations},
  year      = {2015}
}

@inproceedings{li2020gradient,
  title     = {Gradient Descent with Early Stopping Is Provably Robust to Label Noise for Overparameterized Neural Networks},
  author    = {Mingchen Li and Mahdi Soltanolkotabi and Samet Oymak},
  booktitle = {Proc. Int. Conf. Artif. Intell. Stat.},
  pages     = {4313--4324},
  year      = {2020}
}

@inproceedings{song2019selfie,
  title     = {SELFIE: Refurbishing Unclean Samples for Robust Deep Learning},
  author    = {Hwanjun Song and Minseok Kim and Jae-Gil Lee},
  booktitle = {Proc. Int. Conf. Mach. Learn.},
  year      = {2019}
}

@inproceedings{tanaka2018joint,
  title     = {Joint Optimization Framework for Learning with Noisy Labels},
  author    = {Daiki Tanaka and Daiki Ikami and Toshihiko Yamasaki and Kiyoharu Aizawa},
  booktitle = {Proc. IEEE Conf. Comput. Vis. Pattern Recognit.},
  pages     = {5552--5560},
  year      = {2018}
}

@inproceedings{imagenet,
  author    = {Jia Deng and Wei Dong and Richard Socher and Li-Jia Li and Kai Li and Li Fei-Fei},
  title     = {ImageNet: A Large-Scale Hierarchical Image Database},
  booktitle = {Proc. IEEE Conf. Comput. Vis. Pattern Recognit.},
  pages     = {248--255},
  year      = {2009}
}

@inproceedings{Wang-2019-SCE,
  author    = {Yisen Wang and Xingjun Ma and Zaiyi Chen and Yuan Luo and Jinfeng Yi and James Bailey},
  title     = {Symmetric Cross Entropy for Robust Learning with Noisy Labels},
  booktitle = {Proc. IEEE Int. Conf. Comput. Vis.},
  pages     = {322--330},
  year      = {2019}
}

@inproceedings{Ma-2020-APL,
  author    = {Xingjun Ma and Hanxun Huang and Yisen Wang and Simone Romano and Sarah Erfani and James Bailey},
  title     = {Normalized Loss Functions for Deep Learning with Noisy Labels},
  booktitle = {Proc. Int. Conf. Mach. Learn.},
  articleno = {607},
  year      = {2020}
}

@inproceedings{Menon-2020-CanGradClip,
  author    = {Aditya Krishna Menon and Ankit Singh Rawat and Sashank J. Reddi and Sanjiv Kumar},
  title     = {Can Gradient Clipping Mitigate Label Noise?},
  booktitle = {Proc. Int. Conf. Learn. Representations},
  year      = {2020}
}

@inproceedings{Feng-2020-TaylorCE,
  author    = {Lei Feng and Senlin Shu and Zhuoyi Lin and Fengmao Lv and Li Li and Bo An},
  title     = {Can Cross Entropy Loss Be Robust to Label Noise?},
  booktitle = {Proc. Int. Joint Conf. Artif. Intell.},
  articleno = {305},
  year      = {2021}
}

@article{Zhou-2023-asymmetric,
  author    = {Xiong Zhou and Xianming Liu and Deming Zhai and Junjun Jiang and Xiangyang Ji},
  title     = {Asymmetric Loss Functions for Noise-Tolerant Learning: Theory and Applications},
  journal   = {{IEEE} Trans. Pattern Anal. Mach. Intell.},
  volume    = {45},
  number    = {7},
  pages     = {8094--8109},
  year      = {2023}
}

@inproceedings{Kim-2019-NLNL,
  author    = {Youngdong Kim and Junho Yim and Juseung Yun and Junmo Kim},
  title     = {NLNL: Negative Learning for Noisy Labels},
  booktitle = {Proc. IEEE Int. Conf. Comput. Vis.},
  pages     = {101--110},
  year      = {2019}
}

@inproceedings{Kim-2021-JNPL,
  author    = {Youngdong Kim and Juseung Yun and Hyounguk Shon and Junmo Kim},
  title     = {Joint Negative and Positive Learning for Noisy Labels},
  booktitle = {Proc. IEEE Conf. Comput. Vis. Pattern Recognit.},
  pages     = {9437--9446},
  year      = {2021}
}

@article{Wang-2018-impor-rewei,
  author    = {Ruxin Wang and Tongliang Liu and Dacheng Tao},
  title     = {Multiclass Learning with Partially Corrupted Labels},
  journal   = {{IEEE} Trans. Neural Netw. Learn. Syst.},
  volume    = {29},
  number    = {6},
  pages     = {2568--2580},
  year      = {2018}
}

@inproceedings{Chang-2017-active-bias,
  author    = {Haw-Shiuan Chang and Erik Learned-Miller and Andrew McCallum},
  title     = {Active Bias: Training More Accurate Neural Networks by Emphasizing High Variance Samples},
  booktitle = {Adv. Neural Inf. Process. Syst.},
  volume    = {30},
  year      = {2017}
}

@article{Song-2022-noisy-survey,
  title     = {Learning from Noisy Labels with Deep Neural Networks: A Survey},
  author    = {Hwanjun Song and Minseok Kim and Dongmin Park and Yooju Shin and Jae-Gil Lee},
  journal   = {{IEEE} Trans. Neural Netw. Learn. Syst.},
  year      = {2022}
}

@article{zheng2017metaLC,
  title     = {Meta Label Correction for Noisy Label Learning},
  author    = {Guoqing Zheng and Ahmed Hassan Awadallah and Susan Dumais},
  journal   = {Proc. AAAI Conf. Artif. Intell.},
  pages     = {11053--11061},
  year      = {2021}
}

@article{yu2024fatal,
  author    = {X. Yu and S. Zhang and L. Jia and Y. Wang and M. Song and Z. Feng},
  title     = {Noise Is the Fatal Poison: A Noise-Aware Network for Noisy Dataset Classification},
  journal   = {Neurocomputing},
  year      = {2024}
}

@inproceedings{Bekker2016unreliable,
  author    = {A. J. Bekker and J. Goldberger},
  title     = {Training Deep Neural-Networks Based on Unreliable Labels},
  booktitle = {Proc. IEEE Int. Conf. Acoust. Speech Signal Process.},
  year      = {2016}
}

@inproceedings{Misra2016humancentr,
  author    = {Ishan Misra and C. Lawrence Zitnick and M. Mitchell and R. Girshick},
  title     = {Seeing Through the Human Reporting Bias: Visual Classifiers from Noisy Human-Centric Labels},
  booktitle = {Proc. IEEE Conf. Comput. Vis. Pattern Recognit.},
  year      = {2016}
}

@inproceedings{Liu2024temporal,
  title     = {Towards Robust Temporal Activity Localization Learning with Noisy Labels},
  author    = {Daizong Liu and Xiaoye Qu and Xiang Fang and Jianfeng Dong and Pan Zhou and Guoshun Nan and Keke Tang and Wanlong Fang and Yu Cheng},
  booktitle = {Proc. Joint Int. Conf. Comput. Linguistics, Lang. Resour. Eval.},
  pages     = {16630--16642},
  year      = {2024},
  publisher = {ELRA and ICCL}
}

@inproceedings{DBLP:conf/nips/XiaL0WGL0TS20:PDN,
  author    = {Xiaobo Xia and Tongliang Liu and Bo Han and Nannan Wang and Mingming Gong and Haifeng Liu and Gang Niu and Dacheng Tao and Masashi Sugiyama},
  title     = {Part-Dependent Label Noise: Towards Instance-Dependent Label Noise},
  booktitle = {Proc. Adv. Neural Inf. Process. Syst.},
  year      = {2020}
}

@inproceedings{DBLP:conf/iclr/WeiZ0L0022:CIFAR-10N/-100N,
  author    = {Jiaheng Wei and Zhaowei Zhu and Hao Cheng and Tongliang Liu and Gang Niu and Yang Liu},
  title     = {Learning with Noisy Labels Revisited: {A} Study Using Real-World Human Annotations},
  booktitle = {Proc. Int. Conf. Learn. Representations},
  year      = {2022}
}

@article{han2020survey,
  title     = {A Survey of Label-Noise Representation Learning: Past, Present and Future},
  author    = {Bo Han and Quanming Yao and Tongliang Liu and Gang Niu and Ivor W. Tsang and James T. Kwok and Masashi Sugiyama},
  journal   = {arXiv:2011.04406},
  year      = {2020}
}

@inproceedings{arpit2017closer,
  title     = {A Closer Look at Memorization in Deep Networks},
  author    = {Devansh Arpit and Stanis{\l}aw Jastrz{\k{e}}bski and Nicolas Ballas and David Krueger and Emmanuel Bengio and Maxinder S. Kanwal and Tegan Maharaj and Asja Fischer and Aaron Courville and Yoshua Bengio and others},
  booktitle = {Proc. Int. Conf. Mach. Learn. },
  pages     = {233--242},
  year      = {2017},
  organization = {PMLR}
}

@article{ye2023active,
  title     = {Active Negative Loss Functions for Learning with Noisy Labels},
  author    = {Xichen Ye and Xiaoqiang Li and Tong Liu and Yan Sun and Weiqin Tong and others},
  journal   = {Adv. Neural Inf. Process. Syst.},
  volume    = {36},
  pages     = {6917--6940},
  year      = {2023}
}

@inproceedings{DBLP:conf/miccai/medicalNoiseSegmentation,
  author    = {{\'{A}}lvaro Gonz{\'{a}}lez{-}Jim{\'{e}}nez and Simone Lionetti and Philippe Gottfrois and Fabian Gr{\"{o}}ger and Marc Pouly and Alexander A. Navarini},
  title     = {Robust T-Loss for Medical Image Segmentation},
  booktitle = {Proc. Int. Conf. Med. Image Comput. Comput. Assist. Intervent.},
  volume    = {14222},
  pages     = {714--724},
  year      = {2023},
  publisher = {Springer}
}

@inproceedings{long2015fullySemanticSegmentation,
  title     = {Fully Convolutional Networks for Semantic Segmentation},
  author    = {Jonathan Long and Evan Shelhamer and Trevor Darrell},
  booktitle = {Proc. IEEE Conf. Comput. Vis. Pattern Recognit.},
  pages     = {3431--3440},
  year      = {2015}
}

@inproceedings{miccai2015unet,
  title     = {U-Net: Convolutional Networks for Biomedical Image Segmentation},
  author    = {Olaf Ronneberger and Philipp Fischer and Thomas Brox},
  booktitle = {Proc. Int. Conf. Med. Image Comput. Comput. Assist. Intervent.},
  pages     = {234--241},
  year      = {2015},
  organization = {Springer}
}

@inproceedings{he2017maskrcnn,
  title     = {Mask R-CNN},
  author    = {Kaiming He and Georgia Gkioxari and Piotr Doll{\'{a}}r and Ross Girshick},
  booktitle = {Proc. IEEE Int. Conf. Comput. Vis.},
  pages     = {2961--2969},
  year      = {2017}
}

@article{isic,
  author    = {Noel C. F. Codella and David A. Gutman and M. Emre Celebi and Brian Helba and Michael A. Marchetti and Stephen W. Dusza and Aadi Kalloo and Konstantinos Liopyris and Nabin K. Mishra and Harald Kittler and Allan Halpern},
  title     = {Skin Lesion Analysis Toward Melanoma Detection: A Challenge at the 2017 International Symposium on Biomedical Imaging (ISBI), Hosted by the International Skin Imaging Collaboration (ISIC)},
  journal   = {{IEEE} Trans. Biomed. Eng.},
  volume    = {65},
  number    = {7},
  pages     = {168--172},
  year      = {2018},
  publisher = {{IEEE}}
}

@article{DBLP:conf/miccai/LiGH21,
  author    = {Shuailin Li and Zhitong Gao and Xuming He},
  title     = {Superpixel-Guided Iterative Learning from Noisy Labels for Medical Image Segmentation},
  journal   = {{IEEE} Trans. Med. Imaging},
  volume    = {40},
  number    = {9},
  pages     = {525--535},
  year      = {2021},
  publisher = {Springer}
}

@article{nnunet,
  author    = {Fabian Isensee and Paul F. Jaeger and Simon A. A. Kohl and Jens Petersen and Klaus H. Maier-Hein},
  title     = {nnU-Net: A Self-Configuring Method for Deep Learning-Based Biomedical Image Segmentation},
  journal   = {Nat. Methods},
  volume    = {18},
  number    = {2},
  pages     = {203--211},
  year      = {2021},
  publisher = {Nature Publishing Group}
}

@inproceedings{li2021contrastive,
  author    = {Yunfan Li and Peng Hu and Zitao Liu and Dezhong Peng and Joey Tianyi Zhou and Xi Peng},
  title     = {Contrastive Clustering},
  booktitle = {Proc. AAAI Conf. Artif. Intell.},
  volume    = {35},
  number    = {10},
  pages     = {8547--8555},
  year      = {2021}
}

@inproceedings{DBLP:conf/nips/LiuNRF20:ELR,
  author       = {Sheng Liu and
                  Jonathan Niles{-}Weed and
                  Narges Razavian and
                  Carlos Fernandez{-}Granda},
  title        = {Early-Learning Regularization Prevents Memorization of Noisy Labels},
  booktitle    = {Proc. Adv. Neural Inf. Process. Syst.},
  year         = {2020},
}

@inproceedings{DBLP:conf/aaai/OGC,
  author       = {Xichen Ye and
                  Yifan Wu and
                  Weizhong Zhang and
                  Xiaoqiang Li and
                  Yifan Chen and
                  Cheng Jin},
  editor       = {Toby Walsh and
                  Julie Shah and
                  Zico Kolter},
  title        = {Optimized Gradient Clipping for Noisy Label Learning},
  booktitle    = {Proc. AAAI Conf. Artif. Intell.},
  pages        = {9463--9471},
  year         = {2025},
}

@inproceedings{
    ye2025towards,
    title={Towards Robust Influence Functions with Flat Validation Minima},
    author={Xichen Ye and Yifan Wu and WEIZHONG ZHANG and Cheng Jin and Yifan Chen},
    booktitle={Proc. Int. Conf. Mach. Learn. },
    year={2025},
}

@inproceedings{
cho2025dual,
title={Lightweight Dataset Pruning without Full Training via Example Difficulty and Prediction Uncertainty},
author={Yeseul Cho and Baekrok Shin and Changmin Kang and Chulhee Yun},
booktitle={Proc. Int. Conf. Mach. Learn. },
year={2025},
}

@InProceedings{CA2C,
    author    = {Sheng, Mengmeng and Sun, Zeren and Zhou, Tianfei and Shu, Xiangbo and Pan, Jinshan and Yao, Yazhou},
    title     = {CA2C: A Prior-Knowledge-Free Approach for Robust Label Noise Learning via Asymmetric Co-learning and Co-training},
    booktitle = {Proc. IEEE Int. Conf. Comput. Vis.},
    month     = {October},
    year      = {2025},
    pages     = {901-911}
}

@inproceedings{sheng2024enhancing,
    title={Enhancing Robustness in Learning with Noisy Labels: An Asymmetric Co-Training Approach},
    author={Mengmeng Sheng and Zeren Sun and Gensheng Pei and Tao Chen and Haonan Luo and Yazhou Yao},
    booktitle={Proc. ACM Int. Conf. Multimedia},
    year={2024},
}

\begin{IEEEbiographynophoto}{Xichen Ye}
received the B.E. degree in computer science and technology with the School of Information Science and Technology, Hangzhou Normal University, Zhejiang, China, in 2021, and the M.E. degree in computer science and technology with the School of Computer Engineering and Science, Shanghai University, Shanghai, China, in 2024. He is currently a research assistant at the School of Computer Science, Fudan University, Shanghai, China. His research interests include robust machine learning within the field of computer vision.
\end{IEEEbiographynophoto}

\begin{IEEEbiographynophoto}{Yifan Wu} 
received the B.E. degree in intelligent science and technology from the School of Computer Engineering and Science, Shanghai University, Shanghai, China, in 2024. He is currently a research assistant at the School of Computer Science, Fudan University, Shanghai, China. His research interests include anomaly detection, out-of-distribution detection, and noisy label learning.
\end{IEEEbiographynophoto}

\begin{IEEEbiographynophoto}{Yiqi Wang} 
is currently pursuing the B.E. degree in Artificial Intelligence from the School of Computer Engineering and Science, Shanghai University, Shanghai, China. Her research interests include computer vision, data-efficient learning.
\end{IEEEbiographynophoto}

\begin{IEEEbiographynophoto}{Xiaoqiang Li} (Member, IEEE) 
received the Ph.D. degree in computer science from Fudan University, Shanghai, China, in 2004. He is currently an Associate Professor of computer science with Shanghai University, China. He is the Deputy Director of the Multimedia Special Committee of the Shanghai Computer Society. His current research interests include image processing, pattern recognition, computer vision, and machine learning. He has published over 100 conference and journal papers in these areas, including CVPR, ICCV, Nerurips, ICLR, IEEE TCSVT, IEEE TMM, IEEE TIP, \etc
\end{IEEEbiographynophoto}

\begin{IEEEbiographynophoto}{Weizhong Zhang}
received the B.S. and Ph.D. degrees from Zhejiang University in 2012 and 2017, respectively. He is currently a tenure-track Professor with the School of Data Science, Fudan University. His research interests include sparse neural network training, robustness, and out-of-distribution generalization.
\end{IEEEbiographynophoto}

\begin{IEEEbiographynophoto}{Yifan Chen} (Member, IEEE)
received the B.S. degree from Fudan University, Shanghai, China, in 2018, and the PhD degree in Statistics from University of Illinois Urbana-Champaign in 2023. He is currently an assistant professor in computer science and math at Hong Kong Baptist University. He is broadly interested in developing efficient machine learning algorithms, encompassing both statistical and deep learning models.
\end{IEEEbiographynophoto}

\vfill

\appendices
\onecolumn
\clearpage
\pagestyle{plain}  
\setcounter{page}{1}  

\bigskip
\begin{center}
\setlength{\baselineskip}{1.6\baselineskip}
{\LARGE\bf Supplementary Material for ``\mytitle''}
\end{center}

\section{Loss functions}
\label{appendix: loss_functions}

\subsection{Normalized Negative Loss Functions}
\label{appendix: NNLF}
Here, we show how to derive Normalized Negative Loss Functions (NNLFs) into its proper form.
\begin{align}
    \begin{aligned}
        NNLF
        & = \frac{\sum_{k=1}^K (1-\boldsymbol{q}(k|\boldsymbol{x})) (A - \mathcal{L}(f(\boldsymbol{x}),k))}{\sum_{j=1}^K \sum_{k=1}^K (1 - \boldsymbol{q}(k=j|\boldsymbol{x})) (A - \mathcal{L}(f(\boldsymbol{x}),k))} \\
        & = \frac{\sum_{k \ne y} A - \mathcal{L} (f(\boldsymbol{x}),k)}{\sum_{j=1}^K \sum_{k \ne j} A - \mathcal{L} (f(\boldsymbol{x}),k)} \\
        & = \frac{\sum_{k \ne y} A - \mathcal{L} (f(\boldsymbol{x}),k)}{(K-1)\sum_{k=1}^K A - \mathcal{L} (f(\boldsymbol{x}),k)} \\
        & = \frac{1}{K-1} \cdot \Big( 1 - \frac{A-\mathcal{L}(f(\boldsymbol{x}),y)}{\sum_{k=1}^K A - \mathcal{L} (f(\boldsymbol{x}),k)} \Big) \\
        & \propto 1 - \frac{A-\mathcal{L}(f(\boldsymbol{x}),y)}{\sum_{k=1}^K A - \mathcal{L} (f(\boldsymbol{x}),k)}.
    \end{aligned} 
\end{align}

\section{Noise tolerant}
\label{appendix: noise tolerant}

\begin{theoremappendix}
\label{thmapp: symmetric}
Normalized negative loss function $\mathcal{L}_\text{nn}$ is symmetric.
\end{theoremappendix}
\begin{proof}
For all $\boldsymbol{x} \in \mathcal{X}$ and all $f$, we have:
\begin{align}
    \begin{aligned}
        \sum_{k=1}^K \mathcal{L}_\text{nn}(f(\boldsymbol{x}),k)
        & = \sum_{k=1}^K \Big( 1 - \frac{A-\mathcal{L}(f(\boldsymbol{x}),k)}{\sum_{j=1}^K A - \mathcal{L} (f(\boldsymbol{x}),j)} \Big) \\
        & = K - \frac{\sum_{k=1}^K A-\mathcal{L}(f(\boldsymbol{x}),k)}{\sum_{j=1}^K A - \mathcal{L} (f(\boldsymbol{x}),j)} \\
        & = K-1,
    \end{aligned}
\end{align}
where $K-1$ is a constant and $\mathcal{L}_\text{nn}$ satisfies the definition of the symmetric loss function.
\end{proof}

\begin{theoremappendix}
\label{thmapp: sym-noise-robust}
In a multi-class classification problem, any normalized negative loss function $\mathcal{L}_{nn}$ is noise tolerant under symmetric label noise, if noise rate $\eta< \frac{K-1}{K}$.
\end{theoremappendix}
\begin{proof}
From Lemma 1 in \cite{Ma-2020-APL}, a normalized loss function $\mathcal{L}_{\text{norm}}$ is noise-tolerant under symmetric label noise if the noise rate satisfies $\eta < \frac{K-1}{K}$.
By Theorem 1 in our paper, the normalized negative loss function  $\mathcal{L}_{\text{nn}}$ is symmetric, and the scaled loss $\frac{1}{K-1} \mathcal{L}_\text{nn}$ satisfies the property $\frac{1}{K-1} \mathcal{L}_\text{nn} \in [0,1]$, as required for a normalized loss function in \cite{Ma-2020-APL}.
Substituting $\mathcal{L}_{\text{norm}} = \frac{1}{K-1} \mathcal{L}_{\text{nn}}$ into Lemma 1 of \cite{Ma-2020-APL}, we conclude that $\mathcal{L}_{\text{nn}}$ is noise-tolerant under symmetric noise for $\eta < \frac{K-1}{K}$.
\end{proof}

\begin{theoremappendix}
\label{thmapp: asym-noise-robust}
In a multi-class classification problem, suppose $\mathcal{L}_{nn}$ satisfies $0 \le \mathcal{L}_{nn}(f(\boldsymbol{x}),k) \le 1$, $\forall k \in \{1,\cdots,K\}$.
If $R(f^*)=0$, then, any normalized negative loss function $\mathcal{L}_{nn}$ is noise tolerant under asymmetric label noise when noise rate $\eta_{yk}<1-\eta_y$.
\end{theoremappendix}
\begin{proof}
According to Lemma 2 in \cite{Ma-2020-APL}, a normalized loss function $\mathcal{L}_{\text{norm}}$ is noise-tolerant under asymmetric (class-conditional) label noise if: (1) $R(f^*) = 0$, (2) $0 \leq \mathcal{L}_{\text{norm}}(f(x), k) \leq \frac{1}{K-1}, \forall k$ and (3) the noise rates satisfy $\eta_{yk}<1-\eta_y$.
Assuming that $0 \le \mathcal{L}_{nn}(f(\boldsymbol{x}),k) \le 1$, the scaled normalized negative loss function $\frac{1}{K-1} \mathcal{L}_{\text{nn}}$ satisfies the required boundedness conditions: $0 \leq \frac{1}{K} \mathcal{L}_{\text{nn}}(f(x), k) \leq \frac{1}{K-1}, \forall k$.
Furthermore, $\frac{1}{K-1} \mathcal{L}_{\text{nn}}$ is symmetric (from Theorem 1) and satisfies the general property ($\frac{1}{K-1} \mathcal{L}_\text{nn} \in [0,1]$) of a normalized loss function.
Given that $R(f^*) = 0$, substituting $\mathcal{L}_{\text{norm}} = \frac{1}{K-1} \mathcal{L}_{\text{nn}}$ into Lemma 2 of \cite{Ma-2020-APL} allows us to conclude that $\mathcal{L}_{\text{nn}}$ is noise-tolerant under asymmetric label noise when $\eta_{yk}<1-\eta_y$.
\end{proof}

\begin{theoremappendix}
\label{thmapp: anl-noise-robust}
$\forall \alpha, \beta$, the combination $\mathcal{L}_\text{ANL} = \alpha  \mathcal{L}_\text{norm} + \beta  \mathcal{L}_\text{nn}$ is noise tolerant under both symmetric and asymmetric label noise.
\end{theoremappendix}
\begin{proof}
As stated in Lemma 3 of \cite{Ma-2020-APL}, the weighted combination $\alpha \mathcal{L}_\text{Active} + \beta \mathcal{L}_\text{Passive}$ of two noise-tolerant loss functions $\mathcal{L}_\text{Active}$ and $\mathcal{L}_\text{Passive}$ remains noise-tolerant under both symmetric and asymmetric label noise.
From Lemma 1 and Lemma 2 in \cite{Ma-2020-APL}, the normalized loss function $\mathcal{L}_\text{norm}$ is known to be noise-tolerant.
Moreover, Theorem 2 and Theorem 3 in our paper establish that the normalized negative loss function $\mathcal{L}_\text{nn}$ is also noise-tolerant.
By substituting $\mathcal{L}_{\text{Active}} = \mathcal{L}_{\text{norm}}$ and $\mathcal{L}_{\text{Passive}} = \mathcal{L}_{\text{nn}}$ into Lemma 3 of \cite{Ma-2020-APL}, we conclude that the combined loss $\mathcal{L}_\text{ANL} = \alpha  \mathcal{L}_\text{norm} + \beta  \mathcal{L}_\text{nn}$ is noise-tolerant under both symmetric and asymmetric label noise.
\end{proof}

\section{Gradient analysis}
\label{appendix: gradient}
\setcounter{subsection}{0}

\subsection{Gradient of MAE}
The complete derivation of the Mean Absolute Error (MAE) with respect to the classifier's output probabilities is as follows:

In the case of $j \ne y$:
\begin{align}
    \begin{aligned}
        \frac{\partial \mathcal{L}_\text{MAE}}{\partial \boldsymbol{p}(j|\boldsymbol{x})}
        & = \frac{\partial}{\partial \boldsymbol{p}(j|\boldsymbol{x})} \sum_{k=1}^K |\boldsymbol{p}(k|\boldsymbol{x}) - \boldsymbol{q}(k|\boldsymbol{x})| \\
        & = \frac{1}{\partial \boldsymbol{p}(j|\boldsymbol{x})} \big( (1-\boldsymbol{p}(y|\boldsymbol{x})) +\sum_{k \ne y} \boldsymbol{p}(k|\boldsymbol{x}) \big) \\
        & = 1.
    \end{aligned}
\end{align}

In the case of $j = y$:
\begin{align}
    \begin{aligned}
        \frac{\partial \mathcal{L}_\text{MAE}}{\partial \boldsymbol{p}(j|\boldsymbol{x})}
        & = \frac{\partial}{\partial \boldsymbol{p}(j|\boldsymbol{x})} \sum_{k=1}^K |\boldsymbol{p}(k|\boldsymbol{x}) - \boldsymbol{q}(k|\boldsymbol{x})| \\
        & = \frac{1}{\partial \boldsymbol{p}(j|\boldsymbol{x})} \big( (1-\boldsymbol{p}(y|\boldsymbol{x})) +\sum_{k \ne y} \boldsymbol{p}(k|\boldsymbol{x}) \big) \\
        & = -1.
    \end{aligned}
\end{align}

\subsection{Gradient of NNCE}

The complete derivation of our proposed Normalized Negative Cross Entropy (NNCE) with respect to the classifier's output probabilities is as follows:

In the case of $j \ne y$:
\begin{align}
    \begin{aligned}
        \frac{\partial \mathcal{L}_\text{NNCE}}{\partial \boldsymbol{p}(j|\boldsymbol{x})}
        & = \frac{\partial}{\partial \boldsymbol{p}(j|\boldsymbol{x})} \Big( 1 - \frac{ A - (- \log{\boldsymbol{p}(y|\boldsymbol{x})}) }{ \sum^K_{k=1} A - (- \log{\boldsymbol{p}(k|\boldsymbol{x})}) } \Big) \\
        & = - \frac{\partial}{\partial \boldsymbol{p}(j|\boldsymbol{x})} \Big( \frac{A+\log \boldsymbol{p}(y|\boldsymbol{x})}{ \sum^K_{k=1} A + \log{\boldsymbol{p}(k|\boldsymbol{x})} } \Big) \\
        & = - \frac{0-(A+\log \boldsymbol{p}(y|\boldsymbol{x}))\frac{1}{\boldsymbol{p}(j|\boldsymbol{x})}}{\big(\sum^K_{k=1} A + \log{\boldsymbol{p}(k|\boldsymbol{x})}\big)^2} \\
        & = \frac{1}{\boldsymbol{p}(j|\boldsymbol{x})} \cdot \frac{A+\log\boldsymbol{p}(y|\boldsymbol{x})}{\big(\sum^K_{k=1} A + \log{\boldsymbol{p}(k|\boldsymbol{x})}\big)^2}.
    \end{aligned}
\end{align}

In the case of $j = y$:
\begin{align}
    \begin{aligned}
        \frac{\partial \mathcal{L}_\text{NNCE}}{\partial \boldsymbol{p}(j|\boldsymbol{x})}
        & = \frac{\partial}{\partial \boldsymbol{p}(j|\boldsymbol{x})} \Big( 1 - \frac{ A - (- \log{\boldsymbol{p}(y|\boldsymbol{x})}) }{ \sum^K_{k=1} A - (- \log{\boldsymbol{p}(k|\boldsymbol{x})}) } \Big) \\
        & = - \frac{\partial}{\partial \boldsymbol{p}(j|\boldsymbol{x})} \Big( \frac{A+\log \boldsymbol{p}(y|\boldsymbol{x})}{ \sum^K_{k=1} A + \log{\boldsymbol{p}(k|\boldsymbol{x})} } \Big) \\
        & = - \frac{\frac{1}{\boldsymbol{p}(y|\boldsymbol{x})}\big(\sum^K_{k=1} A + \log{\boldsymbol{p}(k|\boldsymbol{x})}\big)-(A+\log \boldsymbol{p}(y|\boldsymbol{x}))\frac{1}{\boldsymbol{p}(y|\boldsymbol{x})}}{\big(\sum^K_{k=1} A + \log{\boldsymbol{p}(k|\boldsymbol{x})}\big)^2} \\
        & = - \frac{1}{\boldsymbol{p}(y|\boldsymbol{x})} \cdot \frac{\sum_{k \ne y} A + \log{\boldsymbol{p}(k|\boldsymbol{x})}}{\big(\sum^K_{k=1} A + \log{\boldsymbol{p}(k|\boldsymbol{x})}\big)^2}.
    \end{aligned}
\end{align}

\subsection{Properties of NNCE}
\setcounter{theoremappendix}{4}

\begin{theoremappendix}
    Given the classifier's output probability $\boldsymbol{p}(\cdot|\boldsymbol{x})$ for sample $\boldsymbol{x}$ and normalized negative cross entropy $\mathcal{L}_\text{NNCE}$. If $\boldsymbol{p}(j_1|\boldsymbol{x})$ $<$ $\boldsymbol{p}(j_2|\boldsymbol{x})$, $j_1 \ne j_2 \ne y$, then \[\frac{\partial \mathcal{L}_\text{NNCE}}{\partial \boldsymbol{p}(j_1|\boldsymbol{x})} > \frac{\partial \mathcal{L}_\text{NNCE}}{\partial \boldsymbol{p}(j_2|\boldsymbol{x})}.\]
\end{theoremappendix}
\begin{proof}
As per the assumption that \(\boldsymbol{p}(j_1|x) < \boldsymbol{p}(j_2|x)\), we have:
\begin{align}
    \frac{1}{\boldsymbol{p}(j_1|x)} > \frac{1}{\boldsymbol{p}(j_2|x)}.
\end{align}
Therefore, we have:
\begin{align}
    \begin{aligned}
    \frac{\partial \mathcal{L}_\text{NNCE}}{\partial \boldsymbol{p}(j_1|x)} & = \frac{1}{\boldsymbol{p}(j_1|x)} \cdot \frac{A + \log \boldsymbol{p}(y|x)}{\left( \sum_{k=1}^K A + \log \boldsymbol{p}(k|x) \right)^2} \\
    & > \frac{1}{\boldsymbol{p}(j_2|x)} \cdot \frac{A + \log \boldsymbol{p}(y|x)}{\left( \sum_{k=1}^K A + \log \boldsymbol{p}(k|x) \right)^2} \\
    & = \frac{\partial \mathcal{L}_\text{NNCE}}{\partial \boldsymbol{p}(j_2|x)}.
    \end{aligned}
\end{align}
This completes the proof.
\end{proof}

\begin{theoremappendix}
    Given the classifier's output probabilities $\boldsymbol{p}(\cdot|\boldsymbol{x}_1)$ and $\boldsymbol{p}(\cdot |\boldsymbol{x}_2)$ of sample $\boldsymbol{x}_1$ and $\boldsymbol{x}_2$, where $\boldsymbol{p}(y|\boldsymbol{x}_1)$ $\ge$ $\boldsymbol{p}(k|\boldsymbol{x}_1)$, $\boldsymbol{p}(y|\boldsymbol{x}_2)$ $\ge$ $\boldsymbol{p}(k|\boldsymbol{x}_2)$, $\forall k \in \{1, \cdots, K \}$, $\boldsymbol{p}(j|\boldsymbol{x}_1)$ $=$ $\boldsymbol{p}(j|\boldsymbol{x}_2)$, and normalized negative cross entropy $\mathcal{L}_\text{NNCE}$.
    If $\boldsymbol{p}(y|\boldsymbol{x}_1)$ $>$ $\boldsymbol{p}(y|\boldsymbol{x}_2)$, $j \ne y$, 
    and $\boldsymbol{p}(k|\boldsymbol{x}_1) \le \boldsymbol{p}(k|\boldsymbol{x}_2)$, $\forall k $ $\in$ $\{1,$ $\cdots,$ $K\}$ $\setminus \{j,y\}$, 
    then \[\frac{\partial \mathcal{L}_\text{NNCE}}{\partial \boldsymbol{p}(j|\boldsymbol{x}_1)}>\frac{\partial \mathcal{L}_\text{NNCE}}{\partial \boldsymbol{p}(j|\boldsymbol{x}_2)}.\]
\end{theoremappendix}
\begin{proof}
    We first consider two functions $f_1$ and $f_2$. The function $f_1$ is defined as follows:
    \begin{equation}
        f_1(\boldsymbol{p}) = \sum_{k=1}^K \log \boldsymbol{p}_k,
    \end{equation}
    where the input $\boldsymbol{p}$ is a discrete probability distribution, $\sum_{k=1}^K \boldsymbol{p}_k=1$ and $0 \le\boldsymbol{p}_k \le 1$, $\forall k$ $\in$ $\{1,$ $\cdots,$ $K\}$.
    Given a discrete probability distribution $\boldsymbol{p}$ and a real number $D$ as input, the function $f_2$ is defined as follows:
    \begin{equation}
        f_2(\boldsymbol{p}, D) = \log (\boldsymbol{p}_y+D) + \log \boldsymbol{p}_j + \sum_{k \ne j \ne y} \log (\boldsymbol{p}_k - d_k),
    \end{equation}
    where $\boldsymbol{p}_y \ge \boldsymbol{p}_k$, $\forall k \in \{1, \cdots, K\}$, $0< D \le 1 - \boldsymbol{p}_y$.
    And $\{d_k\},$ $k \in \{1, \cdots, K\}$ $\setminus$ $\{j, y\}$ is a set of random variables which satisfy following conditions: $0 \le \boldsymbol{p}_k - d_k \le \boldsymbol{p}_k$ and $\sum_{k \ne j \ne y} d_k = D$.

    Next, we consider whether $f_1(\boldsymbol{p}) - f_2(\boldsymbol{p}, D) > 0$.
    Given $\boldsymbol{p}$ and $D$, we have,
    \begin{align}
        \label{eq39}
        \begin{aligned}
            f_1(\boldsymbol{p}) - f_2(\boldsymbol{p}, D)
            & = \Big( \sum_{k=1}^K \log \boldsymbol{p}_k \Big) - \Big( \log (\boldsymbol{p}_y+D) + \log \boldsymbol{p}_j + \sum_{k \ne j \ne y} \log (\boldsymbol{p}_k - d_k) \Big) \\
            & = \log \boldsymbol{p}_y - \log (\boldsymbol{p}_y + D) + \sum_{k \ne j \ne y} \log \boldsymbol{p}_k - \log (\boldsymbol{p}_k - d_k) \\
            & = \log \frac{\boldsymbol{p_y}}{\boldsymbol{p}_y + D} + \sum_{k \ne j \ne y} \log \frac{\boldsymbol{p}_k}{\boldsymbol{p}_k - d_k} \\
            & \ge \log \frac{\boldsymbol{p_y}}{\boldsymbol{p}_y + D} + \sum_{k \ne j \ne y} \log \frac{\boldsymbol{p}_y}{\boldsymbol{p}_y - d_k} \\
            & = \log \boldsymbol{p}_y - \log (\boldsymbol{p}_y + D) + \sum_{k \ne j \ne y} \log \boldsymbol{p}_y - \log (\boldsymbol{p}_y - d_k) \\
            & = \Big( \log \boldsymbol{p}_y + \sum_{k \ne j \ne y} \log \boldsymbol{p}_y \Big) - \Big( \log (\boldsymbol{p}_y +D) + \sum_{k \ne j \ne y} \log (\boldsymbol{p}_y - d_k) \Big) \\
            & \ge \Big(\log \boldsymbol{p}_y + \sum_{k \ne j \ne y} \log \boldsymbol{p}_y \Big) - \sup_{\{d_k\}, k \ne j \ne y} \Big( \log (\boldsymbol{p}_y +D) + \sum_{k \ne j \ne y} \log (\boldsymbol{p}_y - d_k) \Big).
        \end{aligned}
    \end{align}
    The first inequality holds because for a fraction $\frac{a}{b} \ge 1$, we always have $\frac{a}{b}$ $\ge$ $\frac{a+c}{b+c}$, where $c \ge 0$.
    
    To get the maximum value of $\log (\boldsymbol{p}_y +D) + \sum_{k \ne j \ne y} \log (\boldsymbol{p}_y - d_k)$, we must solve the following minimization problem subject to constraints:
    \begin{equation}
    \begin{aligned}
        \min_{\{d_k\}, k \neq j \neq y} &\quad - \Big( \log (\boldsymbol{p}_y +D) + \sum_{k \ne j \ne y} \log (\boldsymbol{p}_y - d_k) \Big), \\
        \text{s.t.} &\quad \sum_{k \neq j \neq y} d_k = D.
    \end{aligned}
    \end{equation}
    We can define the Lagrange function $L$ as follows:
    \begin{equation}
        L(d_1, \cdots, d_K, \lambda) = - \Big( \log (\boldsymbol{p}_y +D) + \sum_{k \ne j \ne y} \log (\boldsymbol{p}_y + d_k) \Big) + \lambda \cdot \Big( \sum_{k \ne j \ne y} d_k - D \Big).
    \end{equation}
    Now we can calculate the gradients:
    \begin{equation}
    \begin{cases}
        \frac{\partial L}{\partial d_k} = - \frac{1}{\boldsymbol{p}_y + d_k} + \lambda, & k \ne j \ne y, \\
        \frac{\partial L}{\partial \lambda} = \sum_{k \neq j \neq y} d_k - D.
    \end{cases}
    \end{equation}
    Let $\frac{\partial L}{\partial d_k} = 0$ and $\frac{\partial L}{\partial \lambda} = 0$, we have:
    \begin{equation}
    \begin{cases}
        - \frac{1}{\boldsymbol{p}_y + d_k} + \lambda = 0, & k \ne j \ne y , \\
        \sum_{k \ne j \ne y} d_k - D = 0,
    \end{cases}
    \end{equation}
    and thus, we have:
    \begin{equation}
        d_k = \frac{D}{K-2}, \quad k \ne j \ne y.
    \end{equation}
    So, the minimization value is:
    \[
        - \Big( \log (\boldsymbol{p}_y +D) + \sum_{k \ne j \ne y} \log (\boldsymbol{p}_y - \frac{D}{K-2}) \Big).
    \]
    
    Now, combining with Eq.~\eqref{eq39}, we have:
    \begin{align}
        \begin{aligned}
            f_1(\boldsymbol{p}) - f_2(\boldsymbol{p}, D)
            & \ge \log \boldsymbol{p}_y + \sum_{k \ne j \ne y} \log \boldsymbol{p}_y - \sup_{\{d_k\}, k \ne j \ne y} \Big( \log (\boldsymbol{p}_y +D) + \sum_{k \ne j \ne y} \log (\boldsymbol{p}_y - d_k) \Big) \\
            & = \log \boldsymbol{p}_y + \sum_{k \ne j \ne y} \log \boldsymbol{p}_y - \Big( \log (\boldsymbol{p}_y +D) + \sum_{k \ne j \ne y} \log (\boldsymbol{p}_y - \frac{D}{K-2}) \Big) \\
            & = \log \boldsymbol{p}_y - \log (\boldsymbol{p}_y + D) + (K-2) \Big( \log \boldsymbol{p}_y - \log (\boldsymbol{p}_y - \frac{D}{K-2}) \Big) \\
            & = (K-2) \Big( \log \boldsymbol{p}_y - \log (\boldsymbol{p}_y - \frac{D}{K-2}) \Big) - \Big( \log (\boldsymbol{p}_y + D) - \log \boldsymbol{p}_y \Big) \\
            & = D \cdot \Big( \frac{\log \boldsymbol{p}_y - \log (\boldsymbol{p}_y - \frac{D}{K-2})}{\frac{D}{K-2}} - \frac{\log (\boldsymbol{p}_y + D) - \log \boldsymbol{p}_y}{D} \Big).
        \end{aligned}
    \end{align}
    Recall the nature of the difference, we have $\frac{\log \boldsymbol{p}_y - \log (\boldsymbol{p}_y - \frac{D}{K-2})}{\frac{D}{K-2}} = \frac{d \log x}{dx} \Big|_{x = \boldsymbol{p}_y - \alpha} = \frac{1}{\boldsymbol{p}_y - \alpha}$, where $0 \le \alpha \le \frac{D}{K-2}$, and $\frac{\log (\boldsymbol{p}_y + D) - \log \boldsymbol{p}_y}{D} = \frac{d \log x}{dx} \Big|_{x = \boldsymbol{p}_y + \beta} = \frac{1}{\boldsymbol{p}_y + \beta}$, where $0 \le \beta \le D$.
    Therefore, it follows:
    \begin{equation}
        \frac{\log \boldsymbol{p}_y - \log (\boldsymbol{p}_y - \frac{D}{K-2})}{\frac{D}{K-2}} = \frac{1}{\boldsymbol{p}_y - \alpha} > \frac{1}{\boldsymbol{p}_y} > \frac{1}{\boldsymbol{p}_y + \beta} = \frac{\log (\boldsymbol{p}_y + D) - \log \boldsymbol{p}_y}{D},
    \end{equation}
    which leads to the inequality:
    \begin{align}
        \frac{\log \boldsymbol{p}_y - \log (\boldsymbol{p}_y - \frac{D}{K-2})}{\frac{D}{K-2}} & > \frac{\log (\boldsymbol{p}_y + D) - \log \boldsymbol{p}_y}{D}.
    \end{align}
    Thus, we have:
    \begin{equation}
        f_1(\boldsymbol{p}) > f_2(\boldsymbol{p}, D).
        \label{eq49}
    \end{equation}
    
    Now, considering $\boldsymbol{p}_k = \boldsymbol{p}(k|\boldsymbol{x}_2), \forall k \in \{1, \cdots, K\}$, $D = \boldsymbol{p}(y|\boldsymbol{x}_1) - \boldsymbol{p}(y|\boldsymbol{x}_2)$, and $d_k$ $=$ $\boldsymbol{p}(k|\boldsymbol{x}_2) - \boldsymbol{p}(k|\boldsymbol{x}_1)$, $k \ne j \ne y$, for the derivatives of $f_1(\boldsymbol{p})$ and $f_2(\boldsymbol{p}, D)$, we have:
    \begin{equation}
        f_1(\boldsymbol{p}) = \sum_{k=1}^K \log \boldsymbol{p}_k = \sum_{k=1}^K \log \boldsymbol{p}(k|\boldsymbol{x}_2),
    \end{equation}
    and:
    \begin{equation}
    \begin{aligned}
        f_2(\boldsymbol{p}, D) &= \log (\boldsymbol{p}_y+D) + \log \boldsymbol{p}_j + \sum_{k \ne j \ne y} \log (\boldsymbol{p}_k - d_k) = \sum_{k=1}^K \log \boldsymbol{p}(k|\boldsymbol{x}_1)
    \end{aligned}
    \end{equation}
    Therefore, combining with Eq.~\eqref{eq49}, we obtain:
    \begin{equation}
        \frac{1}{\big( \sum_{k=1}^K A + \log \boldsymbol{p}(k|\boldsymbol{x}_1) \big)^2} > \frac{1}{\big( \sum_{k=1}^K A + \log \boldsymbol{p}(k|\boldsymbol{x}_2) \big)^2} \\
    \end{equation}
    Under the assumption that $\boldsymbol{p}(j|\boldsymbol{x}_1) = \boldsymbol{p}(j|\boldsymbol{x}_2)$, and $\log \boldsymbol{p}(y|\boldsymbol{x}_1) > \log \boldsymbol{p}(y|\boldsymbol{x}_2)$, \ie, $\boldsymbol{p}(y|\boldsymbol{x}_1) > \boldsymbol{p}(y|\boldsymbol{x}_2)$, we have:
    \begin{align}
        \frac{1}{\boldsymbol{p}(j|\boldsymbol{x}_1)} \cdot \frac{A+\log\boldsymbol{p}(y|\boldsymbol{x}_1)}{\big(\sum^K_{k=1} A + \log{\boldsymbol{p}(k|\boldsymbol{x}_1)}\big)^2} > \frac{1}{\boldsymbol{p}(j|\boldsymbol{x}_2)} \cdot \frac{A+\log\boldsymbol{p}(y|\boldsymbol{x}_2)}{\big(\sum^K_{k=1} A + \log{\boldsymbol{p}(k|\boldsymbol{x}_2)}\big)^2}
    \end{align}
    Thus, we conclude:
    \begin{align}
        \frac{\partial \mathcal{L}_\text{NNCE}}{\partial \boldsymbol{p}(j|\boldsymbol{x}_1)} > \frac{\partial \mathcal{L}_\text{NNCE}}{\partial \boldsymbol{p}(j|\boldsymbol{x}_2)}.
    \end{align}
    This completes the proof.
\end{proof}

\newpage
\section{Experiments}
\setcounter{subsection}{0}

\subsection{Empirical Understandings}
\label{appendix: empirical understandings}

\textbf{Active and passive parts separately}

In Table~\ref{table: active passive separately}, we show the results of the active and passive parts separately.
We train NCE and NNCE separately on CIFAR-10 with different symmetric noise rates while keeping the same parameters as ANL-CE.
Specifically, for $\alpha \mathcal{L}_{\text{NCE}}$, we set $\alpha$ to $5.0$ and for $\beta  \mathcal{L}_{\text{NNCE}}$, we set $\beta$ to $5.0$, while $\delta$ is set to $5 \times 10^{-5}$ for both.
As shown in the results, the test set accuracies of NNCE are very close to that of ANL-CE, except for the case with a 0.8 noise rate.
This suggests that NNCE performs well on its own at low noise rates.
However, at very high noise rates, a combination of active losses is needed to achieve better performance. 

\begin{table}[htbp]
\caption{Test accuracies (\%) of different methods on CIFAR-10 datasets with clean and symmetric
($\eta \in \{0.2, 0.4, 0.6, 0.8\}$) label noise. The top-1 best results are in \textbf{bold}.}
\label{table: active passive separately}
\begin{center}
\begin{tabular}{c|c|cccc}
    \toprule
    Methods & Clean & $\eta=0.2$ & $\eta=0.4$ & $\eta=0.6$ & $\eta=0.8$ \\
    \midrule
    NCE & 88.68 & 81.65 & 74.80 & 63.14 & 37.52 \\
    NNCE & 91.51 & 90.09 & 86.91 & 82.16 & 57.06 \\
    ANL-CE & 91.66 & \textbf{90.02} & \textbf{87.28} & \textbf{81.12} & \textbf{61.27} \\
    \bottomrule
\end{tabular}
\end{center}
\end{table}

\textbf{The choice of regularization method}

We also find that the commonly used L2 regularization may struggle to mitigate the overfitting problem.
To address this, we chose to try using other different regularization methods.
We consider 2 other regularization methods: L1 and Jacobian \cite{sokolic2017robust, hoffman2019robust}.
To compare the performance of these methods, we apply them to ANL-CE for training on CIFAR-10 under 0.8 symmetric noise.
For simplicity and fair comparison, we use $\delta$ as the coefficient of the regularization term and consider it as an additional parameter of ANL.
We keep $\alpha$ and $\beta$ the same as in Section~\ref{experiment: benchmarks}, tune $\delta$ for all three methods by following the parameter tuning setting in Appendix~\ref{appendix: benchmark}.
We also train networks without using any regularization method.
The results are reported in Figure~\ref{fig: reg compare}.
As can be observed, among the three regularization methods, only L1 can somewhat better mitigate the overfitting problem.
If not otherwise specified, all ANLs in this paper use L1 regularization. 

\begin{figure*}[htbp]
\centering
\captionsetup[subfigure]{labelformat=parens, labelsep=space, font=scriptsize}
\subfloat[w/o regularizar]{\includegraphics[width=0.25\textwidth]{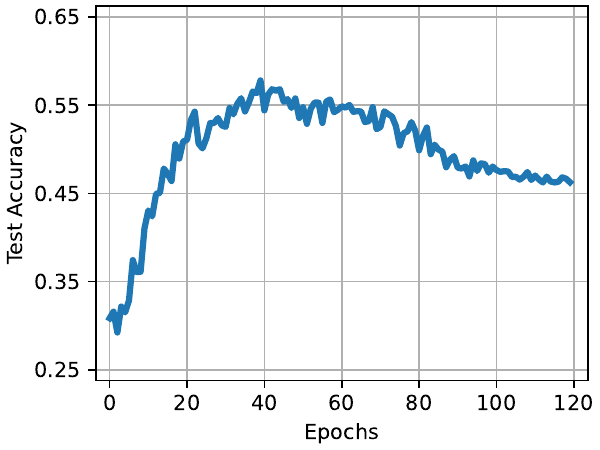}}
\hspace{0.04\textwidth}
\subfloat[L2]{\includegraphics[width=0.25\textwidth]{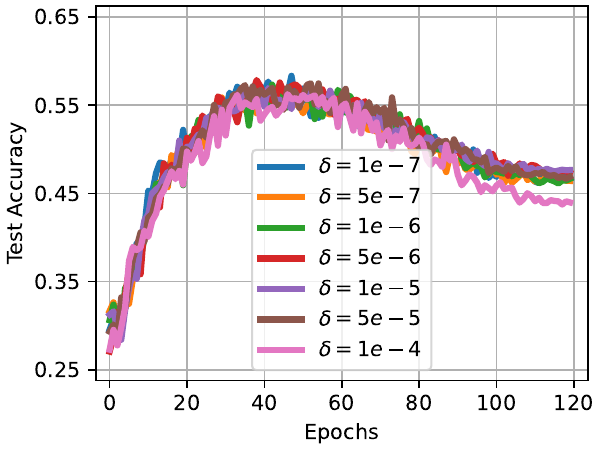}}

\subfloat[Jacobian]{\includegraphics[width=0.25\textwidth]{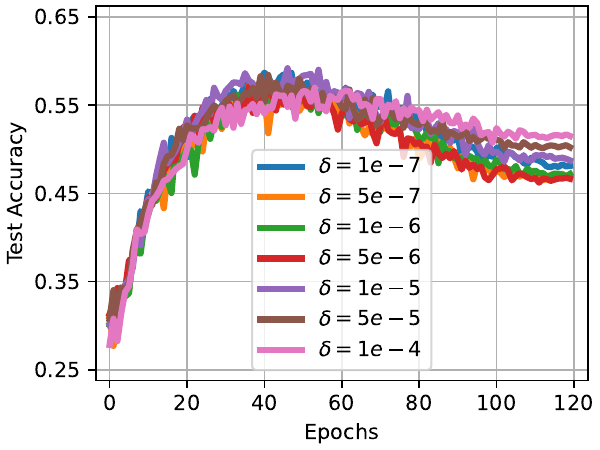}}
\hspace{0.04\textwidth}
\subfloat[L1]
{\includegraphics[width=0.25\textwidth]{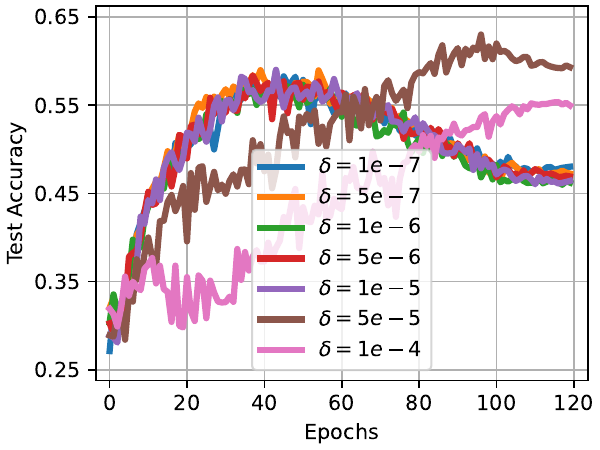}}
\caption{Test accuracies of ANL-CE on CIFAR-10 under 0.8 symmetric noise with different
regularization methods and different parameters. $\delta$ is the weight of regularization term of ANL-CE.}
\label{fig: reg compare}
\end{figure*}

\textbf{Can APL outperform our ANL by changing the regularization method?}
We apply different regularization methods to NCE+RCE for training on CIFAR-10 under 0.8 symmetric noise and compare the results with our ANL-CE.
Specifically, we train networks using loss functions in the form of $\alpha  NCE + \beta  RCE + \delta  \text{Reg}$, where $\alpha, \beta, \delta > 0$ are parameters and $\text{Reg}$ is the regularization term.
Similarly, we consider 3 regularization methods: a) L2, b) L1, and c) Jacobian \cite{sokolic2017robust, hoffman2019robust}.
We keep the $\alpha$ and $\beta$ of NCE+RCE the same as in Section~\ref{experiment: benchmarks} and tune $\delta \in \{1 \times 10^{-7}, 5 \times 10^{-7}, 1 \times 10^{-6}, 5 \times 10^{-6}, 1 \times 10^{-5},5 \times 10^{-5}, 1 \times 10^{-4}\}$ for all three methods.
As can be observed in Figure~\ref{fig: apl reg}, the performance of NCE+RCE can be improved by adapting the regularization method.
However, the improvement from changing the regularization method is limited, and the test accuracies of NCE+RCE are consistently lower than our ANL-CE, no matter how the regularization method is changed and how the parameter $\delta$ is varied.
This verifies that APL cannot outperform our ANL by changing the regularization method and also verifies the performance of our proposed NNLFs.

\begin{figure*}[htbp]
\centering
\captionsetup[subfigure]{labelformat=parens, labelsep=space, font=scriptsize}
\subfloat[L2]{\includegraphics[width=0.25\textwidth]{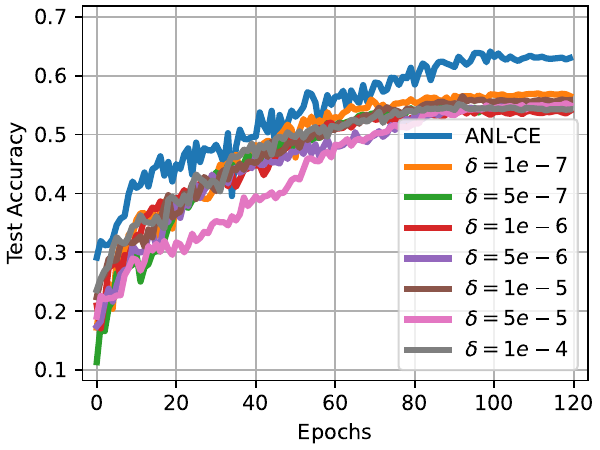}}
\hspace{0.04\textwidth}
\subfloat[L1]{\includegraphics[width=0.25\textwidth]{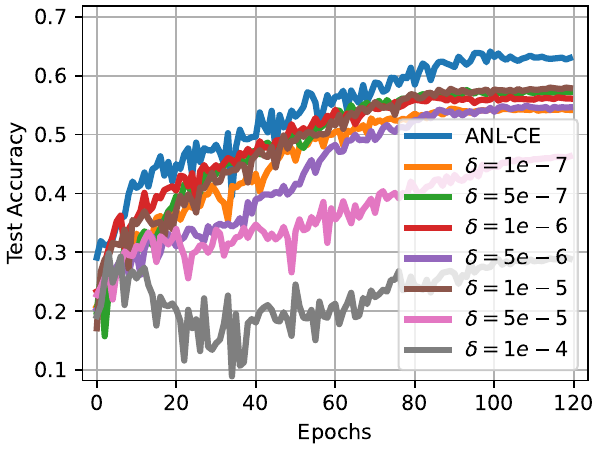}}
\hspace{0.04\textwidth}
\subfloat[Jacobian]{\includegraphics[width=0.25\textwidth]{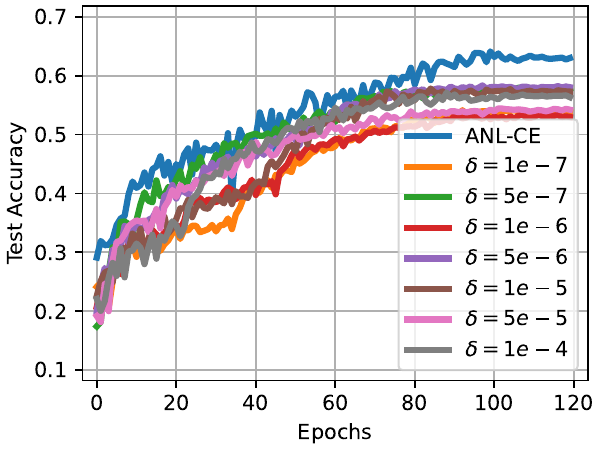}}
\caption{Test accuracies of NCE+RCE on CIFAR-10 under 0.8 symmetric noise with different regularization methods and different parameters. $\delta$ is the weight of regularization term of NCE+RCE. As a comparison, the blue line shows the test accuracy of our ANL-CE with the same data set and noise setting.}
\label{fig: apl reg}
\end{figure*}

\textbf{Gradients of NNCE.}
We conduct experiments with ANL-CE on CIFAR-10 under 0.6 and 0.8 symmetric noise.
For each sample, we calculate the gradient of our NNCE with respect to the predicted probability of each non-labeled class and take their maximum value, which is $\max_{j, j \ne y} \big( \frac{\partial \mathcal{L}_\text{NNCE}}{\partial \boldsymbol{p}(j|\boldsymbol{x})} \big)$.
The results are shown in Figure~\ref{fig: NNCE gradient} in the form of box plots.
As can be observed, as the number of training epochs increases, the gradients of the clean samples are concentrated at larger values and the gradients of the noisy samples are concentrated at smaller values.
The means of the gradients for noisy samples are higher than the medians and smaller than the means for clean samples’ gradients. 
This indicates that although the model is still incorrectly learning some noisy samples, these learned noisy samples constitute only a small fraction of all noisy samples in the training set.
In general, this result verifies that our NNLFs have the ability to focus more on clean samples and ignore noisy samples during the training process, and also verifies our discussion in Section \ref{sec: analysis_NNLFs}.

\begin{figure*}[htbp]
\centering
\captionsetup[subfigure]{labelformat=parens, labelsep=space, font=scriptsize}
\subfloat[CIFAR-10 under 0.6 symmetric noise]{\includegraphics[width=0.7\textwidth]{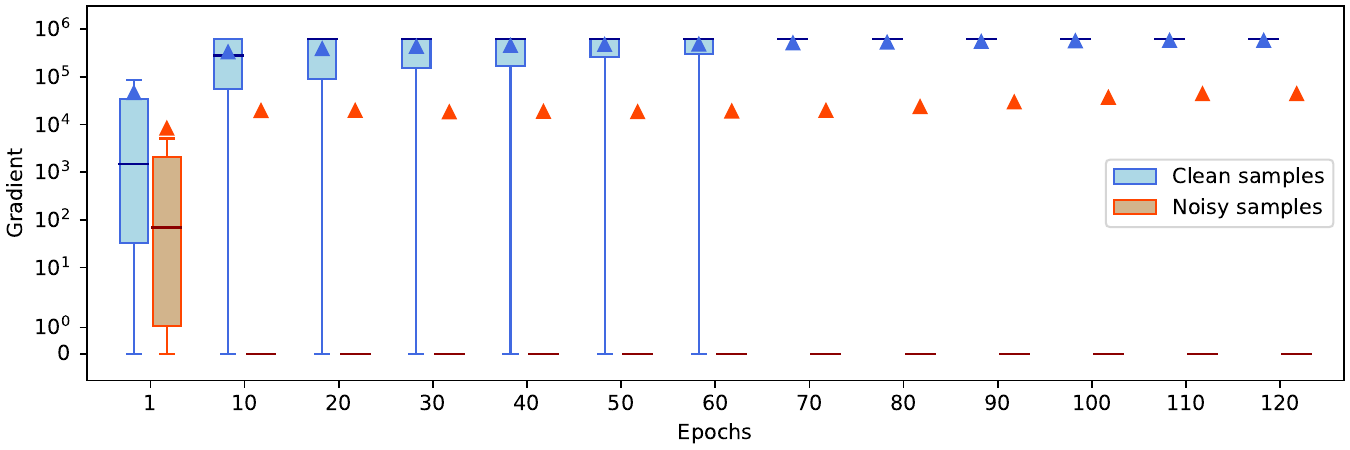}}

\subfloat[CIFAR-10 under 0.8 symmetric noise]{\includegraphics[width=0.7\textwidth]{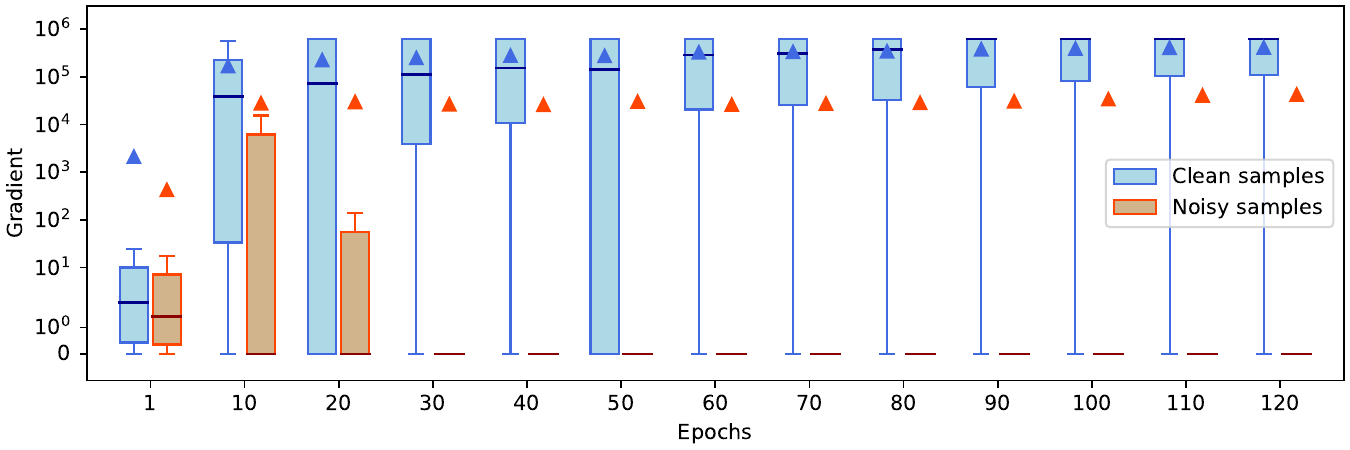}}
\caption{The gradient $\max_{j, j \ne y} \big( \frac{\partial \mathcal{L}_\text{NNCE}}{\partial \boldsymbol{p}(j|\boldsymbol{x})} \big)$ for clean and noisy samples. Blue boxes indicate clean samples, red boxes indicate noisy samples, blue triangles indicate the means of clean samples, and red triangles indicate the means of noisy samples. The outliers are ignored.}
\label{fig: NNCE gradient}
\end{figure*}

\clearpage
\textbf{Parameter Analysis.}
We apply different values of $\alpha$ and $\beta$ to NCE+NNCE for training on CIFAR-10 under 0.8 symmetric noise.
We test the combinations between $\alpha \in \{0.1, 0.5, 1.0, 5.0, 10.0\}$ and $\beta \in \{0.1, 0.5, 1.0, 5.0, 10.0\}$, where $\alpha$ is the weight of NCE and $\beta$ is the weight of NNCE.
The weight decay of NCE+NNCE is set to be the same as NCE+RCE.
As can be observed in Figure~\ref{fig: param analysis}, regardless of how $\alpha$ (the weight of NCE) varies, when $\beta$ (the weight of NNCE) increases, both the robustness and the fitting ability of the model improve, although overfitting occurs.
This verifies that our NNLFs are robust to noisy labels and have good fitting ability.
Although in ANL we use L1 regularization instead of L2 regularization (weight decay), the effect of $\alpha$ and $\beta$ in the loss functions created by ANL should be similar to that in NCE+NNCE.

\begin{figure*}[htbp]
\centering
\captionsetup[subfigure]{labelformat=parens, labelsep=space, font=scriptsize}
\subfloat[\begin{scriptsize}
    $\alpha=0.1, \beta=0.1$
\end{scriptsize}]{\includegraphics[width=0.19\textwidth]{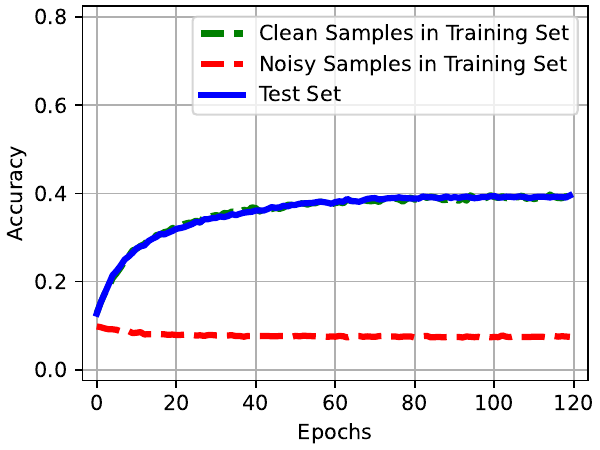}}
\subfloat[\begin{scriptsize}
    $\alpha=0.1, \beta=0.5$
\end{scriptsize}]{\includegraphics[width=0.19\textwidth]{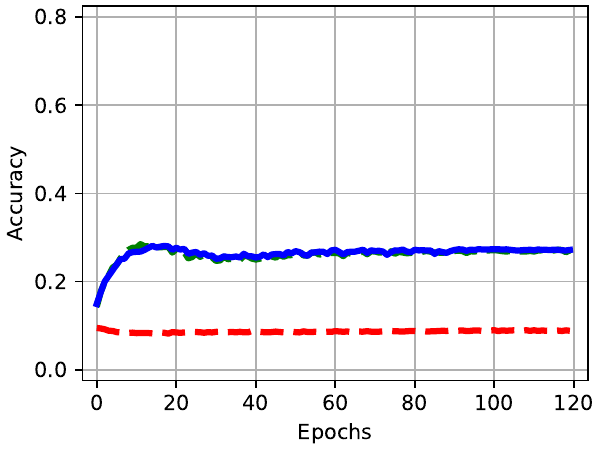}}
\subfloat[\begin{scriptsize}
    $\alpha=0.1, \beta=1.0$
\end{scriptsize}]{\includegraphics[width=0.19\textwidth]{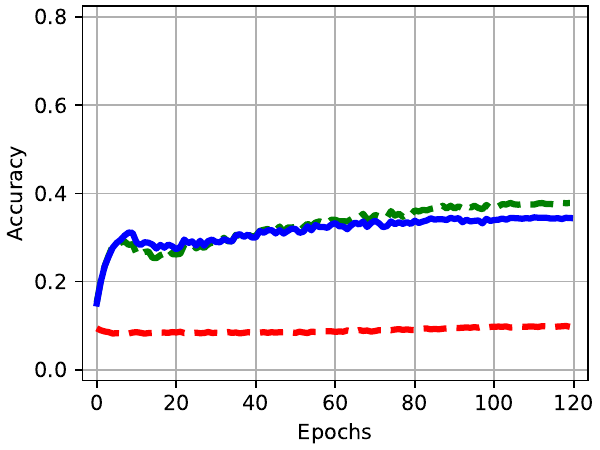}}
\subfloat[\begin{scriptsize}
    $\alpha=0.1, \beta=5.0$
\end{scriptsize}]{\includegraphics[width=0.19\textwidth]{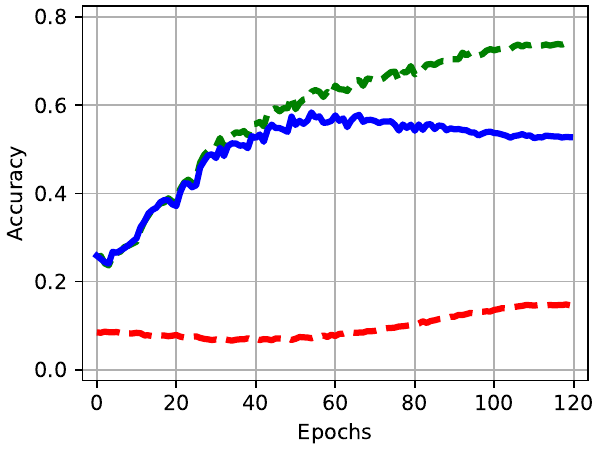}}
\subfloat[\begin{scriptsize}
    $\alpha=0.1, \beta=10.0$
\end{scriptsize}]{\includegraphics[width=0.19\textwidth]{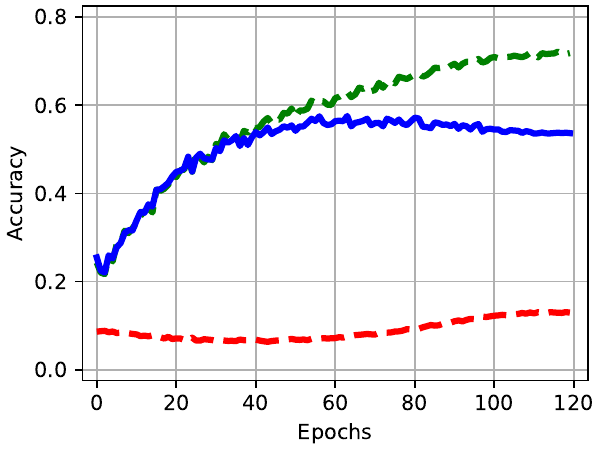}}

\subfloat[\begin{scriptsize}
    $\alpha=0.5, \beta=0.1$
\end{scriptsize}]{\includegraphics[width=0.19\textwidth]{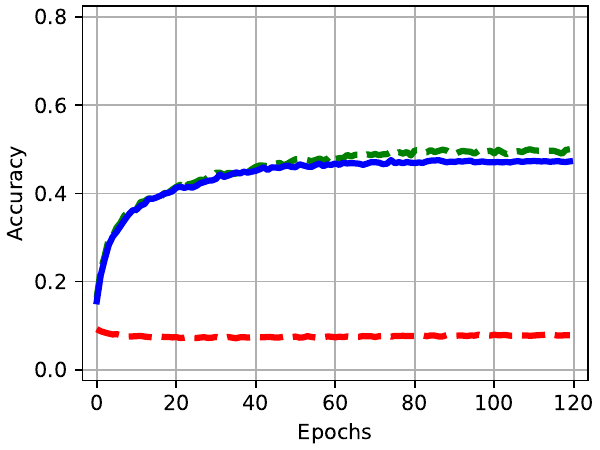}}
\subfloat[\begin{scriptsize}
    $\alpha=0.5, \beta=0.5$
\end{scriptsize}]{\includegraphics[width=0.19\textwidth]{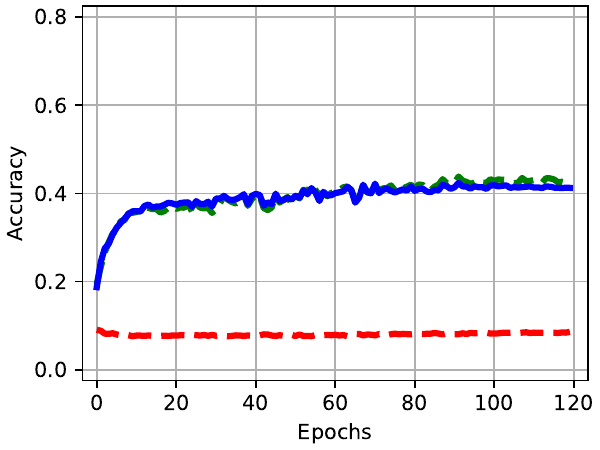}}
\subfloat[\begin{scriptsize}
    $\alpha=0.5, \beta=1.0$
\end{scriptsize}]{\includegraphics[width=0.19\textwidth]{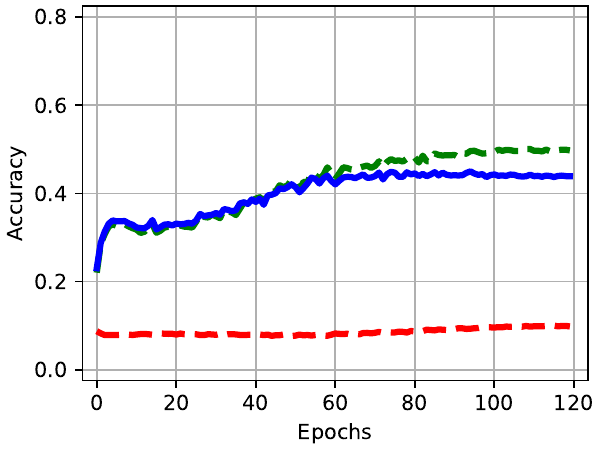}}
\subfloat[\begin{scriptsize}
    $\alpha=0.5, \beta=5.0$
\end{scriptsize}]{\includegraphics[width=0.19\textwidth]{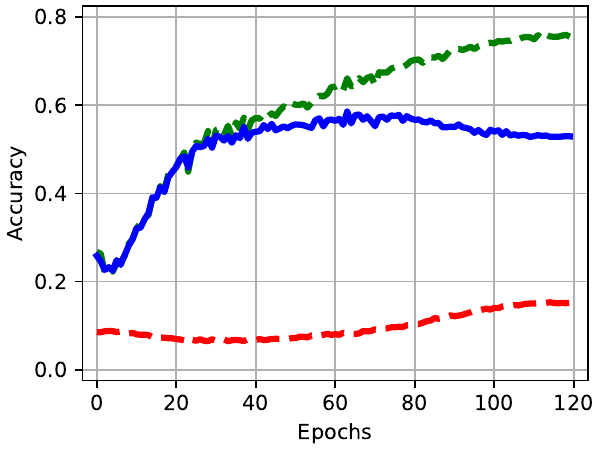}}
\subfloat[\begin{scriptsize}
    $\alpha=0.5, \beta=10.0$
\end{scriptsize}]{\includegraphics[width=0.19\textwidth]{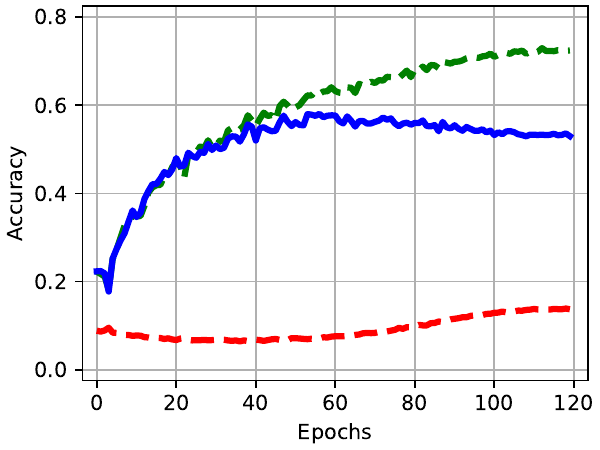}}

\subfloat[\begin{scriptsize}
    $\alpha=1.0, \beta=0.1$
\end{scriptsize}]{\includegraphics[width=0.19\textwidth]{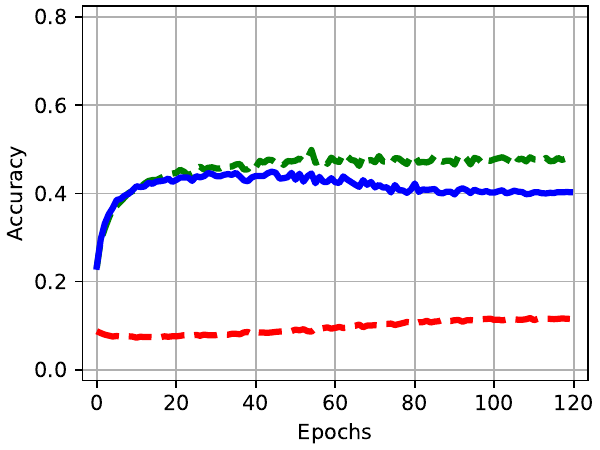}}
\subfloat[\begin{scriptsize}
    $\alpha=1.0, \beta=0.5$
\end{scriptsize}]{\includegraphics[width=0.19\textwidth]{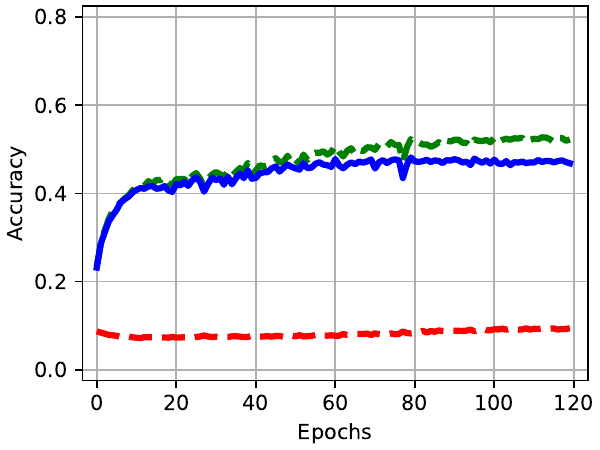}}
\subfloat[\begin{scriptsize}
    $\alpha=1.0, \beta=1.0$
\end{scriptsize}]{\includegraphics[width=0.19\textwidth]{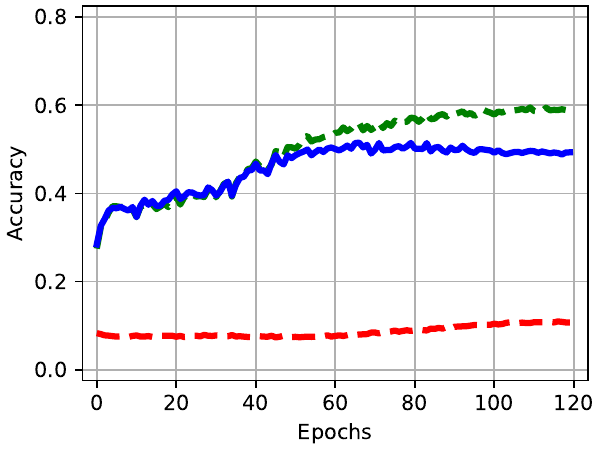}}
\subfloat[\begin{scriptsize}
    $\alpha=1.0, \beta=5.0$
\end{scriptsize}]{\includegraphics[width=0.19\textwidth]{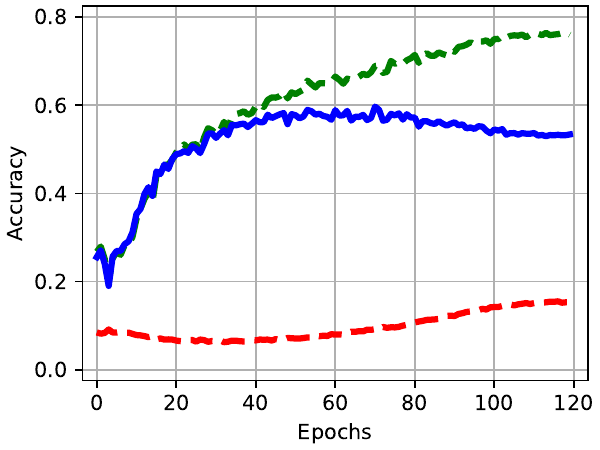}}
\subfloat[\begin{scriptsize}
    $\alpha=1.0, \beta=10.0$
\end{scriptsize}]{\includegraphics[width=0.19\textwidth]{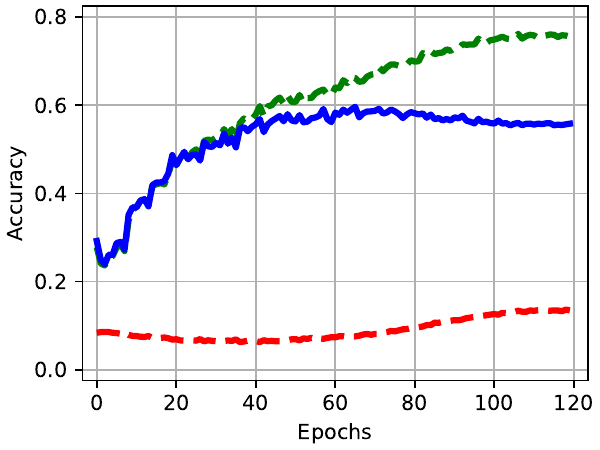}}

\subfloat[\begin{scriptsize}
    $\alpha=5.0, \beta=0.1$
\end{scriptsize}]{\includegraphics[width=0.19\textwidth]{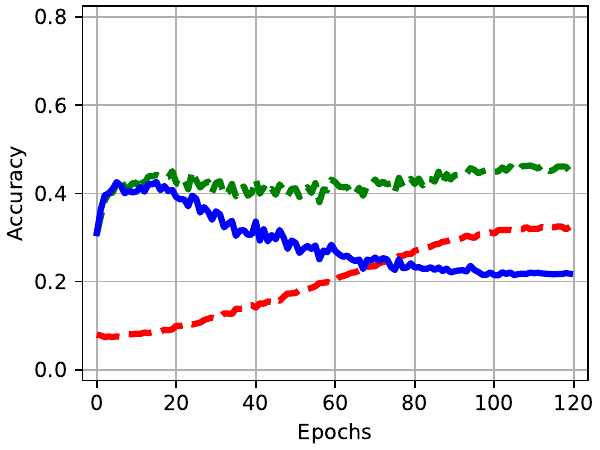}}
\subfloat[\begin{scriptsize}
    $\alpha=5.0, \beta=0.5$
\end{scriptsize}]{\includegraphics[width=0.19\textwidth]{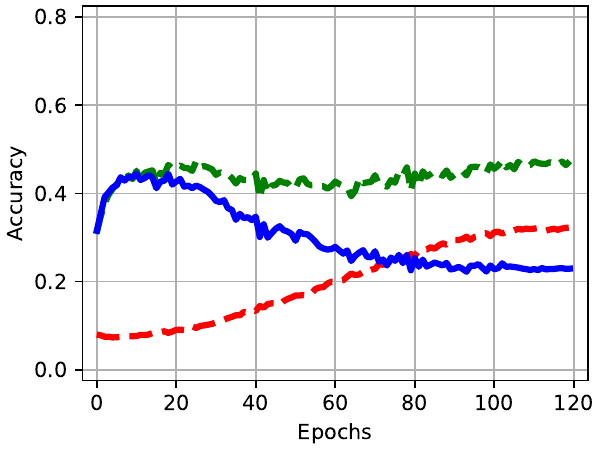}}
\subfloat[\begin{scriptsize}
    $\alpha=5.0, \beta=1.0$
\end{scriptsize}]{\includegraphics[width=0.19\textwidth]{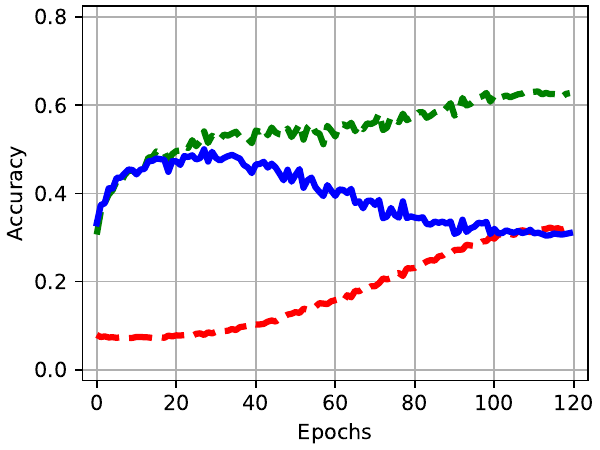}}
\subfloat[\begin{scriptsize}
    $\alpha=5.0, \beta=5.0$
\end{scriptsize}]{\includegraphics[width=0.19\textwidth]{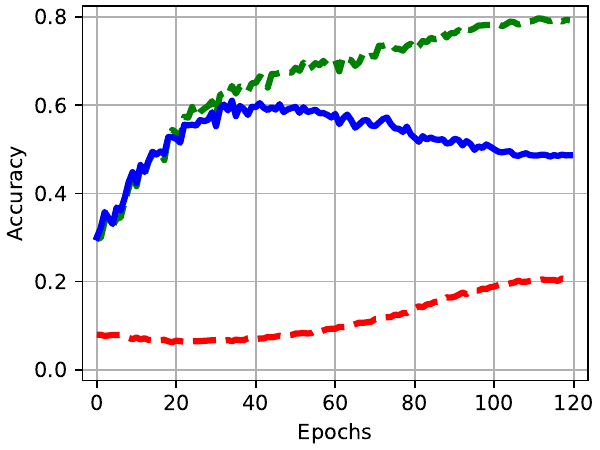}}
\subfloat[\begin{scriptsize}
    $\alpha=5.0, \beta=10.0$
\end{scriptsize}]{\includegraphics[width=0.19\textwidth]{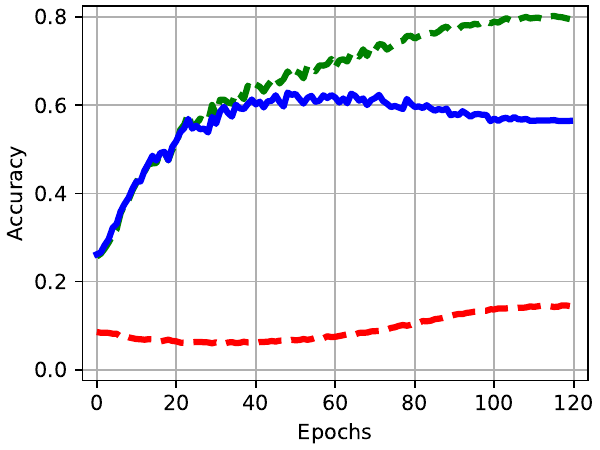}}

\subfloat[\begin{scriptsize}
    $\alpha=10.0, \beta=0.1$
\end{scriptsize}]{\includegraphics[width=0.19\textwidth]{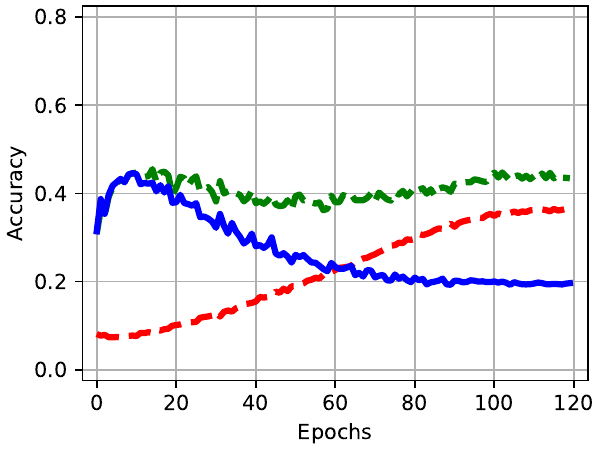}}
\subfloat[\begin{scriptsize}
    $\alpha=10.0, \beta=0.5$
\end{scriptsize}]{\includegraphics[width=0.19\textwidth]{figure/param_analysis/cifar10_sym_0.8_nce_nnce_a10.0_b0.1.pdf}}
\subfloat[\begin{scriptsize}
    $\alpha=10.0, \beta=1.0$
\end{scriptsize}]{\includegraphics[width=0.19\textwidth]{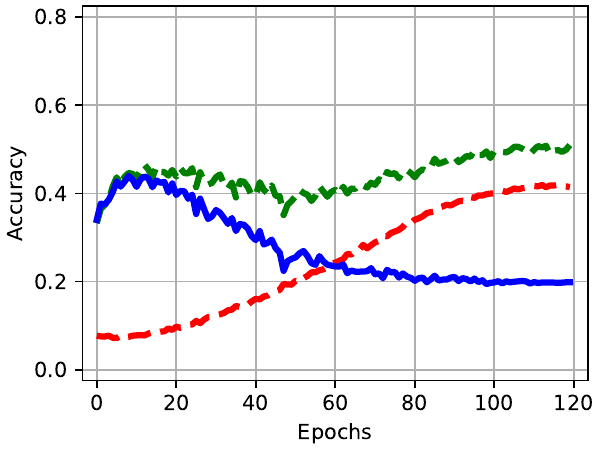}}
\subfloat[\begin{scriptsize}
    $\alpha=10.0, \beta=5.0$
\end{scriptsize}]{\includegraphics[width=0.19\textwidth]{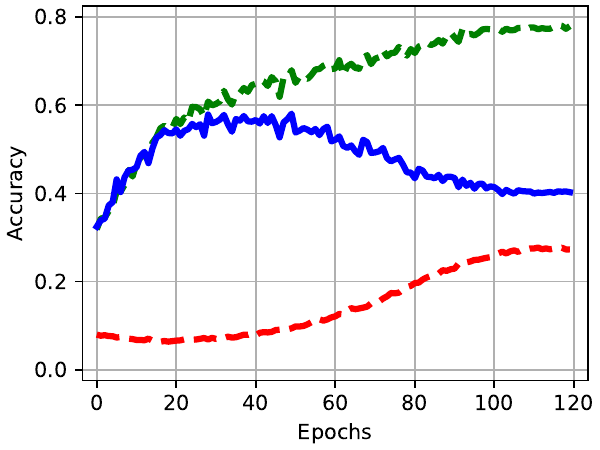}}
\subfloat[\begin{scriptsize}
    $\alpha=10.0, \beta=10.0$
\end{scriptsize}]{\includegraphics[width=0.19\textwidth]{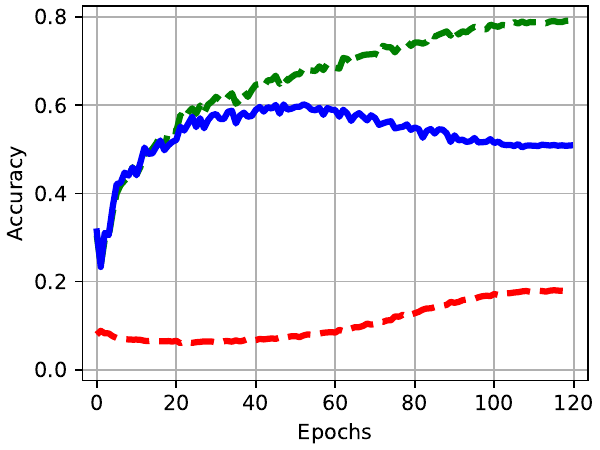}}
\caption{Training and test accuracies of NCE+NNCE on CIFAR-10 under 0.8 symmetric noise with different parameters. $\alpha$ is the weight of NCE and $\beta$ is the weight of NNCE. The accuracies of noisy samples in training set (red dashed line) should be as low as possible, since they are mislabeled.}
\label{fig: param analysis}
\end{figure*}

\clearpage
\subsection{Evaluation on Benchmark Datasets}
\label{appendix: benchmark}

For all baseline methods, the parameters are set to match their original papers.
Detailed parameter settings for CIFAR-10, CIFAR-100 and WebVision can be found in Table~\ref{table: param}.

For Animal-10N dataset, we use L2 regularization (weight decay) for GCE, and L1 regularization for ANL-CE.
We denote the regularization coefficient by $\delta$.
We tune the parameters $\{\delta\}$, $\{q, \delta\}$, $\{\alpha, \beta, \delta\}$ for CE, GCE and ANL-CE respectively.
We use the best parameters $\{1 \times 10^{-3}\}$, $\{0.5, 1 \times 10^{-4}\}$, $\{0.5, 1.0, 1 \times 10^{-6}\}$ for each method in our experiments.

For Clothing-1M dataset, we use L2 regularization (weight decay) for ANL-CE and set the coefficient the same as those for CE, GCE and NCE+RCE.
We tune the parameters \(\{q\}\), $\{\alpha, \beta\}$ and $\{\alpha, \beta\}$ for GCE, NCE+RCE and ANL-CE respectively.
We use the best parameters \(\{0.6\}\), $\{10.0, 1.0\}$ and $\{5.0, 0.1\}$ for each method in our experiments.

\begin{table}[htbp]
\caption{Parameters settings for different methods on CIFAR-10/-100 datasets}
\label{table: param}
\begin{center}
\begin{tabular}{c|ccc}
    \toprule
    Method & CIFAR-10 & CIFAR-100 & WebVision \\
    \midrule
    CE (-) & (-) & (-) & (-)  \\
    FL ($\gamma$) & (0.5) & (0.5) & - \\
    MAE (-) & (-) & (-) & - \\
    GCE ($q$) & (0.7) & (0.7) & (0.7) \\
    SCE ($\alpha, \beta$) & (0.1, 1.0) & (6.0, 0.1) & (10.0, 1.0) \\
    PHCE ($\tau$) & (2) & (8) & -  \\
    TCE ($t$) & (2) & (8) & - \\
    APL ($\alpha, \beta$) & (1.0, 1.0) & (10.0, 0.1) & (50.0, 0.1) \\
    ALF ($\alpha, \beta, a, q$) & (1.0, 4.0, 6.0, 1.5) & (10.0, 0.1, 1.8, 3.0) & (50.0, 0.1, 2.5, 3.0) \\
    ANL-CE ($\alpha, \beta, \delta$) & (5.0, 5.0, $5 \times 10^{-5}$) & (10.0, 1.0, $5 \times 10^{-7}$) & (20.0, 1.0, $6 \times 10^{-6}$) \\
    ANL-FL ($\alpha, \beta, \delta, \gamma$) & (5.0, 5.0, $5 \times 10^{-5}$, 0.5) & (10.0, 1.0, $5 \times 10^{-7}$, 0.5) & (20.0, 1.0, $6 \times 10^{-6}$, 0.5) \\
    ANL-CE* ($\alpha, \beta, \delta, \lambda$) & (5.0, 5.0, $5 \times 10^{-5}$, 2) & (10.0, 1.0, $5 \times 10^{-7}$, 0.01) & - \\
    \bottomrule
\end{tabular}
\end{center}
\end{table}

\end{document}